\documentclass[letterpaper]{article}

\usepackage{aaai}
\usepackage{times}
\usepackage{helvet}
\usepackage{courier}
\newif\ifdraft
\draftfalse

\newif\ifextendedversion
\extendedversiontrue

\usepackage{amsmath,amssymb,amsthm}
\usepackage{color}
\usepackage{xspace}
\usepackage[only bigsqcap]{stmaryrd}
\usepackage{url}
\usepackage{comment}
\usepackage{booktabs}
\usepackage{graphicx}
\usepackage{paralist}
\usepackage{textcomp}

\ifdraft
  \newcommand{\nb}[1]{\textcolor{red}{\bf!}%
  \marginpar[\parbox{15mm}{\raggedleft\scriptsize\textcolor{red}{#1}}]%
   {\parbox{15mm}{\raggedright\scriptsize\textcolor{red}{#1}}}}
\else
  \newcommand{\nb}[1]{}
\fi

\newtheorem{theorem}{Theorem}

\newtheorem{lemma}[theorem]{Lemma}

\newtheorem{definition}[theorem]{Definition}

\theoremstyle{definition}
\newtheorem{example}[theorem]{Example}
\newtheorem{bigexample}[theorem]{Example}

\newcommand{\qedboxfull}{\vrule height 4pt width 4pt depth 0pt}
\newcommand{\qedfull}{\hfill{\qedboxfull}}

\makeatletter
\newcommand{\nobrackettag}[0]{\def\tagform@##1{\maketag@@@{##1}}}
\makeatother

\newcommand{\citeasnoun}[1]{\citeauthor{#1}~\shortcite{#1}}

\newcommand{\A}{\mathcal{A}}

\newcommand{\C}{\mathcal{C}}
\newcommand{\D}{\mathcal{D}}
\newcommand{\E}{\mathcal{E}}
\newcommand{\I}{\mathcal{I}}
\newcommand{\J}{\mathcal{J}}

\renewcommand{\L}{\mathcal{L}}
\newcommand{\M}{\mathcal{M}}

\renewcommand{\O}{\mathcal{O}}
\renewcommand{\P}{\mathcal{P}}
\newcommand{\R}{\mathcal{R}}
\renewcommand{\S}{\mathcal{S}}
\newcommand{\T}{\mathcal{T}}

\newcommand{\cl}{\mathcal}

\newcommand{\tup}[1]{\langle #1\rangle}

\newcommand{\per}{\mbox{\bf .}}

\newcommand{\eg}{e.g.}
\newcommand{\ie}{i.e.}

\newcommand{\ISA}{\sqsubseteq}
\newcommand{\EQU}{\equiv}

\newcommand{\AND}{\sqcap}

\newcommand{\ATMOST}[3]{\ensuremath{\mathop{\leq \! #1} #2 \per #3}}

\newcommand{\SOMET}[1]{\ensuremath{\exists #1}}
\newcommand{\SOME}[2]{\ensuremath{\exists #1 \per #2}}

\newcommand{\ALL}[2]{\ensuremath{\forall{#1}\per{#2}}}
\newcommand{\INV}[1]{#1^-}

\newcommand{\sql}{\mathit{sql}}

\newcommand{\msf}[1]{\mathsf{#1}}
\newcommand{\indivnames}{\msf{N}_{\msf{I}}} 
\newcommand{\conceptnames}{\msf{N}_{\msf{C}}} 
\newcommand{\rolenames}{\msf{N}_{\msf{R}}}
\newcommand{\relationnames}{\msf{N}_{\msf{S}}}

\newcommand{\ind}{\mathsf{Ind}}
\newcommand{\sig}{\mathsf{sig}}

\newcommand{\normexists}{\ensuremath{\mathsf{norm}_{\exists}}}
\newcommand{\normand}{\ensuremath{\mathsf{norm}_{\AND}}}
\newcommand{\etm}{\mathsf{etm}}
\newcommand{\rew}{\mathsf{rew}}
\newcommand{\compile}{\mathsf{comp}}
\newcommand{\rewcomp}{\ensuremath{\mathsf{RewObda}}}
\newcommand{\cut}{\mathsf{cut}}
\newcommand{\adom}{\top\!\!_{\Delta}}

\newcommand{\dom}[1][\I]{\Delta^{#1}}
\newcommand{\Int}[2][\I]{#2^{#1}}

\newcommand{\cert}{\mathit{cert}}
\newcommand{\ans}{\mathit{ans}}

\newcommand{\dlliter}{\textit{DL-Lite}\ensuremath{_{\R}}\xspace}
\newcommand{\el}{\ensuremath{\E\L}\xspace}
\newcommand{\hshiq}{\text{Horn-}\-\ensuremath{\mathcal{SHIQ}}\xspace}
\newcommand{\hshi}{\text{Horn-}\-\ensuremath{\mathcal{SHI}}\xspace}
\newcommand{\shiq}{\ensuremath{\mathcal{SHIQ}}\xspace}
\newcommand{\shi}{\ensuremath{\mathcal{SHI}}\xspace}
\newcommand{\halchiq}{\text{Horn-}\ensuremath{\cl{ALCHIQ}}\xspace}
\newcommand{\halcif}{\text{Horn-}\ensuremath{\cl{ALCIF}}\xspace}
\newcommand{\halcf}{\text{Horn-}\ensuremath{\cl{ALCF}}\xspace}
\newcommand{\owltwo}{OWL\,2\xspace}
\newcommand{\owlql}{OWL\,2\,QL\xspace}

\newcommand{\system}[1]{{\small\textsc{#1}}}
\newcommand{\ontop}{\system{Ontop}\xspace}
\newcommand{\ontoprox}{\system{OntoProx}\xspace}
\newcommand{\clipper}{\system{Clipper}\xspace}

\newcommand{\cq}{CQ\xspace}

\newcommand{\person}{\mathsf{Person}}

\newcommand{\caccount}{\mathsf{CAcc}}
\newcommand{\saccount}{\mathsf{SAcc}}
\newcommand{\name}{\mathsf{inNameOf}}
\newcommand{\entitytab}{\mathsf{ENT}}
\newcommand{\accounttab}{\mathsf{PROD}}

\newcommand{\personview}{V_{\mathsf{Person}}}
\newcommand{\nameview}{V_{\mathsf{inNameOf}}}
\newcommand{\caccountview}{V_{\mathsf{CAcc}}}

\excludecomment{example}

\begin{document}
\ifextendedversion
\title{Beyond \owlql in OBDA: Rewritings and Approximations (Extended Version)}
\else
\title{Beyond \owlql in OBDA: Rewritings and Approximations}
\fi

\author{%
  Elena Botoeva$^{1}$, Diego Calvanese$^{1}$, Valerio Santarelli$^{2}$,
  Domenico F.~Savo$^{2}$,\\[1mm]
  \Large\bf Alessandro Solimando$^{3}$, \and Guohui Xiao$^{1}$\\[2mm]
  $^{1}$ KRDB Research Centre for Knowledge and Data, Free University of
  Bozen-Bolzano, Italy,
  \textit{lastname}\texttt{@inf.unibz.it}\\
  $^{2}$ Dip.\ di Ing.\ Informatica Automatica e Gestionale, Sapienza
  Universit\`a di Roma, Italy,
  \textit{lastname}\texttt{@dis.uniroma1.it}\\
  $^{3}$  DIBRIS, University of Genova, Italy,
  \texttt{alessandro.solimando@unige.it} }

\maketitle

\begin{abstract}
  Ontology-based data access (OBDA) is a novel paradigm facilitating access to
  relational data, realized by linking data sources to an ontology by means of
  declarative mappings.  
\dlliter, which is the logic
  underpinning the W3C ontology language OWL~2 QL and the current language
  of choice for OBDA, has been designed with the goal of
  delegating query answering to the underlying database engine, and
  thus is restricted in expressive power.
  E.g., it does
  not allow one to express disjunctive information, and any form of recursion
  on the data.
  The aim of this paper is to overcome these limitations of \dlliter, and
  extend OBDA to more expressive ontology languages,
  while still leveraging the underlying relational technology for query
  answering.
  We achieve this by relying on two well-known mechanisms, namely conservative
  rewriting and approximation, but significantly extend their practical impact
  by bringing into the picture the mapping, an essential component of OBDA.
  Specifically, we develop techniques to rewrite OBDA specifications with an
  expressive ontology to ``equivalent'' ones with a \dlliter ontology, if
  possible, and to approximate them otherwise.  We do so by exploiting the high
  expressive power of the mapping layer to capture part of the domain semantics
  of rich ontology languages.
  We have implemented our techniques in the prototype system \ontoprox, making
  use of the state-of-the-art OBDA system \ontop and the query answering system
  \clipper, and we have shown their feasibility and effectiveness with
  experiments on synthetic and real-world data.
\end{abstract}

\section{Introduction}
\label{sec:introduction}

Ontology-Based Data Access (OBDA) is a popular pa\-ra\-digm that enables end users
to access data sources through an ontology, abstracting away low-level details
of the data sources themselves.  The ontology provides a high-level description
of the domain of interest, and is semantically lin\-ked to the data sources by
means of a set of mapping assertions \cite{CDLL*09,GSVW*15}.
Typically, the data sources are represented as relational data, the ontology is
constituted by a set of logical axioms over concepts and roles, and each
mapping assertion relates an SQL query over the database to a concept or role
of the ontology.

As an example, consider a bank domain, where we can specify that a checking
account in the name of a person is a simple account by means of the axiom
(expressed in description logic notation)
$
  \caccount \AND \SOME{\name}{\person} \ISA \saccount.
$
We assume that the information about the accounts and their owners is stored in
a database $\D$, and that the \mbox{ontology terms} $\caccount$, $\name$, and
$\person$ are connected to $\D$ respectively via the mapping assertions
$\sql_1(x) \rightsquigarrow \caccount(x)$,
$\sql_2(x,y) \rightsquigarrow \name(x,y)$ and
$\sql_3(x) \rightsquigarrow \person(x)$, where each $\sql_i$ is a (possibly
very complex) SQL query over $\D$.  Suppose now that the user intends to
extract all simple accounts from $\D$.  Formulating
such a query directly over $\D$ would require to know precisely how $\D$ is
structured, and thus could be complicated.  Instead, exploiting OBDA, the user
can simply query the ontology with $q(x)=\saccount(x)$, and rely on the OBDA
system to get the answers.

Making OBDA work efficiently over large amounts of data,
requires that que\-ry answering over the ontology is \emph{first-order
 (FO)-rew\-ri\-table}\footnote{Recall that FO queries constitute the core of
 SQL.} \cite{CDLLR07,ACKZ09}, which in turn limits the expressiveness of the
ontology language, and the degree of detail with which the domain of interest
can be captured.
The current language of choice for OBDA is \dlliter, the logic underlying
\owlql \cite{W3Crec-OWL-Profiles},
which has been specifically designed to ensure FO-rewritability of query
answering.  Hence, it does not allow one to express disjunctive information, or
any form of recursion on the data (\eg, as resulting from qualified
existentials on the left-hand side of concept inclusions), since using such
constructs in general cau\-ses the loss of FO-rewritability \cite{CDLLR13}.
For this reason, in many situations the expressive power of \dlliter is too
restricted to capture real-world scenarios; \eg, the axiom in our example is
not expressible in \mbox{\dlliter}.

The aim of this work is to overcome these limitations of \dlliter by allowing
the use of additional constructs in the ontology.  To be able to exploit the
added value coming from OBDA in real-world settings, an important requirement
is the efficiency of query answering, achieved through a rewriting-based
approach.  This is only possible for ontology languages that are FO-rewritable.
Two general mechanisms that have been proposed to cope with computational
complexity coming from high expressiveness of ontology languages, and that
allow one to regain FO-rewritability, are conservative rewriting \cite{LuPW11}
and approximation \cite{RePZ10b,CMRSS14}.  Given an ontology in a powerful
language, in the former approach it is rewritten, when possible, into an
equivalent one in a restricted language, while in the latter it is
approximated, thus losing part of its semantics.

In this work, we significantly extend the practical impact of both approaches
by bringing into the picture
\emph{the mapping}, an essential component of OBDA that has been ignored so
far. Indeed, it is
a fairly expressive component of an OBDA system, since it allows one to make
use of arbitrary SQL (hence FO) queries to relate the content of the data
source to the elements of the ontology.  Hence, a natural question is how one
can use the mapping component to capture as much as possible additional domain
semantics, resulting in better approximations or more cases where conservative
rewritings are possible, while maintaining a \dlliter ontology.

We illustrate how this can be done on our running examp\-le, where the
non-\dlliter axiom can be encoded by
add\-ing the assertion
$\sql_1(x)\bowtie\sql_2(x,y)\bowtie\sql_3(y) \rightsquigarrow \saccount(x)$ to
the mapping.  This assertion connects $\D$ directly to the ontology term
$\saccount$ by making use of a join of the SQL queries in the original mapping.
We observe that the resulting mapping, together with the ontology in which the
non-\dlliter axiom has been removed, constitutes a conservative rewriting of
the original OBDA specification.

\begin{example}\label{ex:running}
  Consider the following \halchiq TBox $\T^b = \{ ~ \caccount \AND
  \SOME{\name}{\person} \ISA \saccount ~ \}$
  %
  %
  specifying that 
  a checking account in the name of a person is a simple account.
  \qedfull
\end{example}

In this paper, we elaborate on this idea, by introducing a novel
\emph{framework for rewriting and approximation of OBDA specifications}.
Specifically, we provide a notion of \emph{rewriting} based on \emph{query
 inseparability} of OBDA specifications \cite{BiRo15}.  To deal with those
cases where it is not possible to rewrite the OBDA specification into a query
in\-se\-pa\-rab\-le one whose ontology is in \dlliter, we give a notion of
\emph{approximation} that is sound for query answering.
%
We develop techniques for rewriting and approximation of OBDA specifications
based on compiling the extra expressiveness into the mappings.  We target
rather expressive ontology languages, and for \halchiq, a Horn fragment of
\owltwo, we study decidability of existence of OBDA rewritings, and techniques
to compute them when they exist, and to approximate them, otherwise.

We have implemented our techniques in a prototype system called \ontoprox,
which exploits functionalities provided by the \ontop \cite{RoKZ13} and
\clipper systems \cite{EOSTX12} to rewrite or approximate an OBDA specification
expressed in \hshiq to one that can be directly processed by any OBDA system.
We have evaluated \ontoprox over synthetic and real OBDA instances against
\begin{inparaenum}[\it (i)]
\item the default \ontop behavior,
\item local semantic approximation (LSA),
\item global semantic approximation (GSA), and
\item \clipper over materialized ABoxes.
\end{inparaenum}
We observe that using \ontoprox, for a few queries we have been able to obtain
more answers (in fact, complete answers, as confirmed by \clipper).
However, for many queries \ontoprox showed no difference with respect to the
default \ontop behavior.  One reason for this is that in the considered
real-world scenario, the mapping designers put significant effort to manually
create complex mappings that overcome the limitations of \dlliter.  Essentially
they followed the principle of the technique presented here, and therefore
produced an OBDA specification that was already ``complete'' by design.

The observations above immediately suggest a significant practical value of our
approach, which can be used to facili\-tate the design of new OBDA
specifications for existing expressive ontologies: instead of a manual
compilation, which is cumbersome, error-prone, and difficult to maintain,
mapping designers can write straightforward mappings, and the re\-sulting
OBDA specification can then be automatically trans\-formed into a \dlliter
OBDA specification with rich mappings.

The paper is structured as follows. In Section~\ref{sec:preliminaries}, we
provide some preliminary notions, and in Section~\ref{sec:framework}, we
present our framework of OBDA rewriting and approximation. In
Section~\ref{sec:algorithms}, we illustrate a technique for computing the
OBDA-rewriting of a given \halchiq specification.  In
Section~\ref{sec:approximation}, we address the problem of OBDA-rewritability,
and show how to obtain an approximation when a rewriting does not exist.  In
Section~\ref{sec:experiments}, we discuss our prototype \ontoprox and
experiments.  Finally, in Section~\ref{sec:conclusions}, we conclude the paper.
\ifextendedversion The omitted proofs can be found in the appendix.
\else
Omitted proofs can be found in the extended version of this paper
\cite{BCSSSX15}.
\fi

\section{Preliminaries}
\label{sec:preliminaries}


We give some basic notions about ontologies and OBDA.

\subsection{Ontologies}

We assume to have the following pairwise disjoint countably infinite alphabets:
$\conceptnames$ of \emph{concept names}, 
$\rolenames$ of \emph{role names}, 
and $\indivnames$ of constants (also called \emph{individuals}).
We consider ontologies expressed in Description Logics (DLs). 
Here we present the logics \halchiq, the Horn fragment of \shiq without role
transitivity, and \dlliter, for which we develop some of the technical results
in the paper.  However, the general approximation framework is applicable to
any fragment of \owltwo.

A \halchiq TBox in normal form is a finite set of axioms: \emph{concept
 inclusions (CIs)} $\bigsqcap_i A_i \ISA C$, \emph{role inclusions (RIs)}
$R_1\ISA R_2$ and \emph{role disjointness} axioms $R_1 \AND R_2 \ISA \bot$,
where $A$, $A_i$ denote concept names, $R$, $R_1$, $R_2$ denote role names $P$
or their inverses $P^-$, and $C$ denotes a concept of the form $\bot$, $A$,
$\SOME{R}{A}$, $\ALL{R}{A}$, or $\ATMOST{1}{R}{A}$ \cite{Kaza09}.
For an inverse role $R=P^-$, we use $R^-$ to denote $P$.  $\bot$ denotes the
empty concept/role.
A \dlliter TBox is a finite set of axioms of the form $B_1\ISA B_2$,
$B_1\AND B_2\ISA\bot$, $R_1\ISA R_2$, and $R_1\AND R_2\ISA\bot$, where $B_i$
denotes a concept of the form $A$ or $\SOME{R}{\top}$.  In what follows, for
simplicity we write $\SOMET{R}$ instead of $\SOME{R}{\top}$, and we use $N$ to
denote either a concept or a role name.  We also assume that all TBoxes are in
normal form.

An \emph{ABox} is a finite set of \emph{membership assertions} of the form
$A(c)$ or $P(c,c')$, where $c,c'\in\indivnames$.
For a DL $\L$, an \emph{$\L$-ontology} is a pair $\O=\tup{\T,\A}$, where $\T$ is
an $\L$-TBox and $\A$ is an ABox.
A \emph{signature} $\Sigma$ is a finite set of concept and role names. An
ontology $\O$ is said to be defined over (or simply, \emph{over}) $\Sigma$ if
all the concept and role names occurring in it belong to $\Sigma$ (and likewise
for TBoxes, ABoxes, concept inclusions, etc.).  When $\T$ is over $\Sigma$, we
denote by $\sig(\T)$ the subset of $\Sigma$ actually occurring in $\T$.
Moreover we denote with $\ind(\A)$, the set of individuals appearing in $\A$.

The semantics, models, and the notions of satisfaction and consistency of
ontologies are defined in the standard way.  We only point out that we adopt
the \textit{Unique Name Assumption} (UNA), and for simplicity we also assume to
have \emph{standard names}, \ie, for every interpretation $\I$ and every
constant $c\in\indivnames$ interpreted by $\I$, we have that $\Int{c}=c$.

\begin{example}\label{ex:structural}
  Let $\T^b$ be the TBox of Example~\ref{ex:running}. Then $\T^b_1 =
  \mathsf{norm}(\T^b) = \{~ \person \ISA \ALL{\INV{\name}}{A_1}, \caccount \AND
  A_1 \ISA \saccount ~\}$.
  \qedfull
  %
\end{example}


%

\subsection{OBDA and Mappings}

Let $\S$ be a relational schema over a countably infinite set $\relationnames$
of database predicates.  For simplicity, we assume to deal with plain
relational schemas without constraints, and with database instances that
directly store abstract objects (as opposed to values). In other words, a
database instance $\D$ of $\S$ is a set of ground atoms over the predicates
in~$\relationnames$ and the constants in $\indivnames$.\footnote{All our
 results easily extend to the case where objects are constructed from retrieved
 database values \cite{CDLL*09}.}  Queries over $\S$ are expressed in SQL.  We
use $\varphi(\vec{x})$ to denote that query $\varphi$ has
$\vec{x}=x_1,\ldots,x_n$ as \emph{free} (\ie, answer) variables, where $n$ is
the arity of $\varphi$.  Given a database instance $\D$ of $\S$ and a query
$\varphi$ over $\S$, $\ans(\varphi,\D)$ denotes the set of tuples of constants
in $\indivnames$ computed by evaluating $\varphi$ over $\D$.

In OBDA, one provides access to an (external) database through an ontology
TBox, which is connected to the database by means of a mapping. Given a source
schema $\S$ and a TBox $\T$, a (GAV) \emph{mapping assertion} between $\S$ and
$\T$
has the form $ \varphi(x)\leadsto A(x) \text{ or } \varphi'(x,x')\leadsto
P(x,x'), $ where $A$ and $P$ are respectively concept and role names, and
$\varphi(x)$, $\varphi'(x,x')$ are arbitrary (SQL) queries expressed over $\S$.
Intuitively, given a database instance $\D$ of $\S$ and a mapping assertion
$m=\varphi(x)\leadsto A(x)$, the instances of the concept~$A$ generated by $m$
from $\D$ is the set $\ans(\varphi,\D)$; similarly for a mapping assertion
$\varphi(x,x')\leadsto P(x,x')$.

An \emph{OBDA specification} is a triple $\P=\tup{\T,\M,\S}$, where $\T$ is a
DL TBox, $\S$ is a relational schema, and $\M$ is a finite set of mapping
assertions. Without loss of generality, we assume that all concept and role
names appearing in $\M$ are contained in $\sig(\T)$.
An \emph{OBDA instance} is a pair $\tup{\P,\D}$, where $\P$ is an OBDA
specification, and $\D$ is a database instance of $\S$.
%
The semantics of the OBDA instance $\tup{\P,\D}$ is specified in terms of
interpretations of the concepts and roles in $\T$.  We define it by relying on
the following (\emph{virtual}\footnote{We call such an ABox `virtual', because
  we are not interested in actually materializing its facts.}) ABox
\[
  \A_{\M,\D}{\,=\,}\{N(\vec{o}) \mid \vec{o}\in \ans(\varphi,\D) \text{ and }
  \varphi(\vec{x})\leadsto N(\vec{x}) \text{ in\! } \M\}
\]
%
generated by $\M$ from $\D$, where $N$ is a concept or role name in $\T$.
Then, a model of $\tup{\P,\D}$ is simply a model of the ontology
$\tup{\T,\A_{\M,\D}}$.

Following \citeasnoun{DLLM*13}, we split each mapping assertion $m =
\varphi(\vec{x}) \leadsto N(\vec{x})$ in $\M$ into two parts by introducing an
intermediate view name $V_m$ for the SQL query $\varphi(\vec{x})$.  We obtain a
\emph{low-level} mapping assertion of the form $\varphi(\vec{x}) \leadsto
V_m(\vec{x})$, and a \emph{high-level} mapping assertion of the form
$V_m(\vec{x}) \leadsto N(\vec{x})$.
%
In our technical development, we deal only with the high-level mappings. Hence,
we abstract away the low-level mapping part, and in the following we directly
consider the intermediate views as our data sources.

\begin{example}\label{ex:obda}
We introduce the OBDA specification $\P^b=\tup{\T^b,\M^b,\S^b}$, where $\T^b$
is the TBox from Example~\ref{ex:running}, $\S^b$ is the following schema
{\footnotesize
  \[
  \entitytab(\texttt{ID},\texttt{TYPE},\texttt{EMPID}), \qquad
  \accounttab(\texttt{NUM},\texttt{TYPE},\texttt{CUSTID})
  \]
}
and $\M^b$ consists of the following three mapping assertions:
{\footnotesize
   \nobrackettag
  \[
  \begin{array}{@{}l}
    m_{\msf{P}}{:}\, \text{\scriptsize\texttt{SELECT\;ID\;AS\;X\;FROM\;ENT\;WHERE\;ENT.TYPE='P'}}
    \rightsquigarrow \person(\texttt{X})\\
    m_{\msf{N}}{:}\, \text{\scriptsize\texttt{SELECT\;NUM\;AS\;X,\,CUSTID\;AS\;Y\;FROM PROD}}
    \rightsquigarrow \name(\texttt{X},\texttt{Y})\\
    m_{\msf{C}}{:}\, \text{\scriptsize\texttt{SELECT\;NUM\;AS\;X\;FROM\;PROD\;P\;WHERE\;P.TYPE='B'}}
    \rightsquigarrow \caccount(\texttt{X})
  \end{array}
  \]}%
%
The corresponding high-level mapping $\M^b_H$ consists of assertions
$h_{\msf{P}}, h_{\msf{N}}, h_{\msf{C}}$, where:
%
  {\small
    \nobrackettag
    \[
    \begin{array}[b]{@{}l@{~}l}
      h_{\msf{P}}: & \personview(x) \rightsquigarrow  \person(x) \\
      h_{\msf{N}}: & \nameview(x,y) \rightsquigarrow  \name(x,y) \\
      h_{\msf{C}}: & \caccountview(x) \rightsquigarrow  \caccount(x)
    \end{array}
    \tag{\qedboxfull}
    \]}%
\end{example}

\subsection{Query Answering}


We consider conjunctive queries, which are the basic and most important
querying mechanism in relational database systems and ontologies.  A
\emph{conjunctive query} $(\cq)$ $q(\vec{x})$ over a signature $\Sigma$ is a
formula $\exists \vec{y}.\, \varphi(\vec x, \vec y)$, where $\varphi$ is  a
conjunction of atoms
$N(\vec{z})$, such that \mbox{$N$ is a concept or role} name in $\Sigma$, and
$\vec{z}$ are variables from $\vec{x}$ and $\vec{y}$.
The set of \emph{certain answers} to a CQ $q(\vec{x})$ over an ontology
$\tup{\T,\A}$, denoted $\cert(q,\tup{\T,\A})$, is the set of tuples $\vec{c}$
of elements from $\ind(\A)$ of the same length as $\vec{x}$, such that
$q(\vec{c})$ (considered as a FO sentence) holds in every model of
$\tup{\T,\A}$.
%
We 
mention two more query classes.  An \emph{atomic query} (AQ) is a CQ consisting
of exactly one atom whose variables are all free.
A \emph{CQ with inequalities (CQ$^{\neq}$)} is a CQ that may contain inequality
atoms between the variables of the predicate atoms.

Given a CQ $q$, an OBDA specification $\P=\tup{\T,\M,\S}$ and a database
instance $\D$ of $\S$, the answer to $q$ over the OBDA instance $\tup{\P,\D}$,
denoted $\cert(q,\P,\D)$, is defined as $\cert(q,\tup{\T,\A_{\M,\D}})$.
Observe that, when $\D$ is inconsistent with $\P$ (\ie, $\tup{\P,\D}$ does not
have a model), then $\cert(q,\P,\D)$ is the set of all possible tuples of
constants in $\A_{\M,\D}$ (of the same arity as $q$).


\section{An OBDA Rewriting Framework}
\label{sec:framework}

We extend the notion of query inseparability of ontologies \cite{BKRWZ14} to
OBDA specifications.  We adopt the proposal by \citeasnoun{BiRo15}, but we do
not enforce preservation of inconsistency.

\begin{definition}
  Let $\Sigma$ be a signature. Two OBDA specifications
  $\P_1 = \tup{\T_1, \M_1, \S}$ and $\P_2=\tup{\T_2,\M_2,\S}$ are
  $\Sigma$-\emph{CQ inseparable} if $\cert(q,\P_1,\D)=\cert(q,\P_2,\D)$, for
  every CQ $q$ over $\Sigma$ and every database instance $\D$ of $\S$.
\end{definition}

In OBDA, one must deal with the trade-off between the computational complexity
of query answering and the expressiveness of the ontology language.  Suppose
that for an OBDA specification $\P = \tup{\T,\M,\S}$, $\T$ is expressed in an
ontology language $\L$ that does not allow for efficient query answering.  A
possible solution
is to exploit the expressive power of the mapping layer to compute a new OBDA
specification $\P'=\tup{\T',\M',\S}$ in which $\T'$ is expressed in a language
$\L_t$ more suitable for query answering than $\L$. The aim is to encode in
$\M'$ not only $\M$ but also part of the semantics of $\T$, so that $\P'$ is
query-inseparable from $\P$. This leads to the notion of
rewriting of OBDA specifications.

\begin{definition}
  Let $\L_t$ be an ontology language.  The OBDA specification
  $\P'=\tup{\T',\M',\S}$ is a \emph{CQ-rewriting} in $\L_t$ of the OBDA
  specification $\P=\tup{\T,\M,\S}$ if
  \begin{inparaenum}[\it (i)]
  \item $\sig(\T) \subseteq \sig(\T')$,
  \item $\T'$ is an $\L_t$-TBox, and
  \item $\P$ and $\P'$ are $\Sigma$-CQ inseparable, for $\Sigma = \sig(\T)$.
  \end{inparaenum}
  If such $\P'$ exists, we say that $\P$ is \emph{CQ-rewritable into $\L_t$}.
\end{definition}


We observe that the new OBDA specification can be defined over a signature that
is an extension of that of the original TBox. This is specified by
condition~\textit{(i)}. In condition~\textit{(ii)}, we impose that the new
ontology is specified in the target language $\L_t$. Finally,
condition~\textit{(iii)} imposes that the OBDA specifications cannot be
distinguished by CQs over the original TBox.  Note that the definition allows
for changing the ontology and the mappings, but not the source schema,
accounting for the fact that the data sources might not be under the control of
the designer of the OBDA specification.

As expected, it is not always possible to obtain a CQ-rewriting of $\P$ in an
ontology language $\L_t$ that allows for efficient query answering.  Indeed,
the combined expressiveness of $\L_t$ with the new mappings might not be
sufficient to simulate query answering over $\P$ without loss.  In these cases,
we can resort to approximating query answers over $\P$ in a \emph{sound} way,
which means that the answers to queries posed over the new specification are
contained in those produced by querying $\P$.  Hence, we say that the OBDA
specification $\P' = \tup{\T',\M',\S}$ is a \emph{sound CQ-approximation} in
$\L_t$ of the OBDA specification $\P=\tup{\T, \M, \S}$ if $\P'$ satisfies
\textit{(i)}, \textit{(ii)}, and
$\cert(q,\P',\D) \subseteq \cert(q,\P,\D)$, for each CQ $q$ over $\sig(\T)$ and
for each instance $\D$ of $\S$.

Next, we study CQ-rewritability of OBDA specifications into \dlliter,
developing suitable techniques.

\section{Rewriting OBDA Specifications}
\label{sec:algorithms}

In this section, we develop our OBDA rewriting technique,
which relies on Datalog rewritings of the TBox (and mappings).
Recall that a \emph{Datalog program} (with inequalities) is a finite set of
definite \emph{Horn} clauses \emph{without} functions symbols, \ie, rules of
the form $\mathit{head}\leftarrow\varphi$, where $\varphi$ is a finite
non-empty list of predicate atoms and guarded inequalities called the body of
the rule, and $\mathit{head}$ is an atom, called the head of the rule, all of
whose variables occur in the body.  The predicates that occur in rule heads are
called \emph{intensional} (IDB), the other predicates are called
\emph{extensional} (EDB).

\subsection{ET-mappings}
\label{sec:extended-t-mappings}

Now, we extend the notion of T-mappings introduced by \citeasnoun{RoKZ13}, and
define the notion of an ET-mapping that results from compiling into the mapping
the expressiveness of ontology languages that are Datalog rewritable, as
introduced below.

%
We first introduce notation we need.
Let $\Pi$ be a Datalog program and $N$ an IDB predicate.  For a database $\D$
over the EDB predicates of $\Pi$, let $N_\Pi^i(\D)$ denote the set of facts
about $N$ that can be deduced from $\D$ by at most $i \geq 1$ applications of
the rules in $\Pi$, and let $N_\Pi^\infty(\D)=\bigcup_{i \geq 1} N_\Pi^i(\D)$.
It is known that the predicate $N_\Pi^\infty(\cdot)$ defined by $N$ in $\Pi$
can be characterized by a possibly infinite union of CQ$^{\neq}$s
\cite{CGKV88}, \ie, there exist CQ$^{\neq}$s $\varphi^N_0,\varphi^N_1,\ldots$
such that $N_\Pi^\infty(\D) = \bigcup_{i\geq 0} \{N(\vec{a}) \mid
\vec{a}\in\ans(\varphi^N_i,\D)\}$, for every~$\D$.  The $\varphi^N_i$'s are
called the \emph{expansions} of $N$ and can be described in terms of expansion
trees%
\ifextendedversion (see Appendix~\ref{sec:datalog-expansion}).  \else ;
cf. \cite[Appendix~A]{BCSSSX15}.  \fi We denote by $\Phi_\Pi(N)$ the set of
expansion trees for $N$ in $\Pi$, and abusing notation also the (possibly
infinite) union of CQ$^{\neq}$s corresponding to it.
Note that $\Phi_\Pi(N)$ might be infinite due to the presence of IDB predicates
that are \emph{recursive}, \ie, either directly or indirectly refer to
themselves.

We call a TBox $\T$ \emph{Datalog rewritable} if it admits a translation
$\Pi_{\T}$ to Datalog that preserves consistency and answers to AQs (see, \eg,
the translations by \citeasnoun{HuMS05}, \citeasnoun{EOSTX12}, and
\citeasnoun{TSCS15} for \hshiq, and by \citeasnoun{CMSH13} for \shi).  We
assume that $\Pi_{\T}$ makes use of a special nullary predicate $\bot$ that
encodes inconsistency, \ie, for an ABox $\A$, $\tup{\T,\A}$ is consistent iff
$\bot_{\Pi_{\T}}^\infty(\A)$ is empty.\footnote{Here we simply consider $\A$ 
as a database.}  We also assume that $\Pi_\T$ includes the following auxiliary
rules, which ensure that $\Pi_\T$ derives all possible facts constructed over
$\sig(\T)$ and $\ind(\A)$ whenever $\tup{\T,\A}$ is inconsistent:
\[
\begin{array}{@{}l}
\adom(x) \leftarrow A(x); \
\adom(x) \leftarrow P(x,y); \
\adom(y) \leftarrow P(x,y); \\
\hspace*{0.8cm}
A(x) \leftarrow \bot, \adom(x); \ \ \
P(x,y) \leftarrow \bot, \adom(x), \adom(y);
\end{array}
\]
where $A$ and $P$ respectively range over concept and role names in $\sig(\T)$,
and $\adom$ is a fresh unary predicate denoting the set of all the individuals
appearing in $\A$.


In the following, we denote with $\Pi_\M$ the (high-level) mapping $\M$ viewed
as a Datalog program, and with $\Pi_{\T,\M}$ the Datalog program $\Pi_\T \cup
\Pi_\M$ \emph{associated to} a Datalog rewritable TBox $\T$ and a mapping $\M$.
From the properties of the translation $\Pi_\T$ (and
the simple structure of $\Pi_\M$), we obtain that $\Pi_{\T,\M}$ satisfies the
following:
%
\begin{lemma}
  \label{lem:t-m-program}
  Let $\tup{\T, \M, \S}$ be an OBDA specification where $\T$ is Datalog
  rewritable.  Then, for every database instance $\D$ of $\S$, concept or role
  name $N$ of $\T$, and $\vec{a}$ in $\ind(\A_{\M,\D})$, we have that
  $\tup{\T,\A_{\M,\D}} \models N(\vec{a})$ iff
  $N(\vec{a}) \in N_{\Pi_{\T,\M}}^\infty(\D)$.
\end{lemma}

For a predicate $N$, we say that an expansion
$\varphi^N\in\Phi_{\Pi_{\T,\M}}(N)$ is \emph{DB-defined}
if $\varphi^N$ is defined over database predicates.
Now we are ready to define ET-mappings.
\begin{definition}
  Let $\tup{\T, \M, \S}$ be an OBDA specification where $\T$ is Datalog
  rewritable.  The \emph{ET-mapping} for $\M$ and $\T$, denoted $\etm_\T(\M)$,
  is defined as the set of assertions of the form $\varphi^N(\vec{x}) \leadsto
  N(\vec{x})$ such that $N$ is a concept or role name in $\T$, and
  $\varphi^N\in\Phi_{\Pi_{\T,\M}}(N)$ is DB-defined.
  %
  %
\end{definition}

It is easy to show that, for $\M'=\etm_\T(\M)$ and each database instance $\D$,
the virtual ABox $\A_{\M',\D}$ (which can be defined for ET-mappings as for
ordinary mappings) contains all facts entailed by $\tup{\T,\A_{\M,\D}}$.
In this sense, the ET-mapping $\etm_\T(\M)$ plays for a Datalog rewritable TBox
$\T$ the same role as T-mappings play for (the simpler) \dlliter TBoxes.
Note that, in general, an ET-mapping is not a mapping, as it may contain
infinitely many assertions.  However, $\A_{\M',\D}$ is still finite, given that
it is constructed over the finite number of constants appearing in $\D$.
%
\begin{example}\label{ex:extmappings}
  Consider the TBox $\T_1$ in Example~\ref{ex:structural} and the
  mapping $\M^b$ of Example~\ref{ex:obda}.
  It is possible to show that $\etm_{\T_1}(\M^b)$ is a finite set of
  mapping assertions, denoted $\M^b_3$, that extends $\M^b_H$ with the
  assertion $\caccountview(X), \nameview(X,Y), \personview(Y)
  \rightsquigarrow \saccount(X)$.
  \qedfull
  %
\end{example}
%







\subsection{Rewriting \halchiq OBDA Specifications to \dlliter}
\label{sec:rewriting}

\begin{figure}[tb]
  \centering
  \fbox{
   \begin{minipage}{0.96\linewidth}
     \textbf{Input}: \halchiq TBox $\T$ and mapping $\M$.\\
     \textbf{Output}: %
     \dlliter TBox $\T_r$ and ET-mapping $\M_c$.
     \begin{enumerate}[\textbf{Step} 1:\!\!]
     \item $\T_1$ is obtained from $\T$ by adding all CIs of the form
       $\bigsqcap A_i \ISA \SOME{R}{(\bigsqcap A_j')}$ entailed by $\T$, for
       concept names $A_i, A_j' \in \sig(\T)$.
     \item $\T_2 = \normexists(\T_1)$.
     \item $\T_3 = \normand(\T_2)$.
     \item $\M_c$ is $\etm_{\T_3}(\M)$, and $\T_r$ is the \dlliter TBox
       consisting of all \dlliter axioms over $\sig(\T_3)$ entailed by $\T_3$
       (including the trivial ones \mbox{$N\ISA N$}).
     \end{enumerate}
   \end{minipage}}
  \caption{\mbox{\small OBDA specification rewriting algorithm $\rewcomp$.}}
  \vspace{-5mm}
  \label{fig:rewcomp}
\end{figure}

Let $\tup{\T,\M,\S}$ be an OBDA specification, where $\T$ is a \halchiq TBox
over a signature $\Sigma$.  Figure~\ref{fig:rewcomp} describes the algorithm
$\rewcomp(\T,\M)$, which constructs a \dlliter TBox $\T_r$ and an ET-mapping
$\M_c$ such that $\tup{\T_r,\M_c,\S}$ is $\Sigma$-CQ inseparable from
$\tup{\T, \M, \S}$.

In Step~2, the algorithm applies to $\T_1$ the normalization procedure
$\normexists$, which gets rid of concepts of the form
$\SOME{R}{(\bigsqcap A_j')}$ in the right-hand side of CIs. This is achieved by
 the following well-known substitution \cite{ACKZ09}: every CI
$\bigsqcap_{i=1}^m A_i \ISA \SOME{R}{(\bigsqcap_{j=1}^n A_j')}$ in $\T_1$ is
replaced with $\bigsqcap_{i=1}^m A_i \ISA \SOMET{P_{\mathit{new}}}$,
$P_{\mathit{new}} \ISA R$, and $\top \ISA \ALL{P_{\mathit{new}}}{A'_j}$, for
$1\leq j \leq n$, where $P_{\mathit{new}}$ is a fresh role name.
Notice that the latter two forms of inclusions introduced by $\normexists$ are
actually in \dlliter, as $\top \ISA \ALL{P_{\mathit{new}}}{A'_j}$ is equivalent
to $\SOMET{P_{\mathit{new}}^-}\ISA A'_j$.
In Step~3, the algorithm applies to $\T_2$ a further normalization procedure,
$\normand$, which introduces a fresh concept name $A_{A_1\AND\cdots\AND A_n}$
for each concept conjunction $A_1\AND\cdots\AND A_n$ appearing in $\T_2$, and
adds $A_1 \AND \cdots \AND A_n \EQU A_{A_1 \AND \cdots \AND A_n}$\footnote{We
 use `$\EQU$' to abbreviate inclusion in both directions.} to the TBox.
Note that $\normexists(\T_1)$ and $\normand(\T_2)$ are model-conservative
extensions of $\T_1$ and $\T_2$, respectively \cite{LuWW07}, as one can easily
show.
%
We denote by $\rew(\T)$ the resulting TBox $\T_r$, which in general is
exponential in the size of $\T$, and by $\compile(\T,\M)$ the resulting
ET-mapping $\M_c$, which in general is infinite.

\begin{example}\label{ex:andtbox}
  Consider the TBox $\T^b$ from Example~\ref{ex:running} and the mapping $\M^b$
  from Example~\ref{ex:obda}. Observe that the intermediate TBoxes $\T^b_1$,
  $\T^b_2$ and $\T^b_3$ obtained during $\rewcomp(\T^b,\M^b)$ coincide with the
  TBox $\T^b_1$ from Example~\ref{ex:structural}.  Moreover, observe that the
  intermediate $\M^b_3$ is exactly $\M^b_3$ from
  Example~\ref{ex:extmappings}. Now, in step 4, the TBox $\T^b_4$ is computed
  by adding to $\T^b_3$ the assertion $A_{\caccount \AND A_1} \ISA \saccount$.
  %
  Consequently, the mapping $\M^b_4$ is computed by adding to $\M^b_3$ the
  high-level mapping assertion:
  $\caccountview(x),\,\nameview(x,y),\,\personview(y) \rightsquigarrow
  A_{\caccount \AND A_1}(x)$
  \qedfull
  %
\end{example}

\begin{bigexample}
  Assume that the domain knowledge is represented by the axiom about bank
  accounts from Section~\ref{sec:introduction}.  The normalization of this
  axiom is the TBox
  $\T^b = \{ \person \ISA \ALL{\INV{\name}}{A_1}, \caccount \AND A_1 \ISA
  \saccount \}$.
  Assume that the database schema $\S^b$ consists of the two relations
  {\small $\entitytab(\texttt{ID},\texttt{TYPE},\texttt{EMPID})$,
   $\accounttab(\texttt{NUM},\texttt{TYPE},\texttt{CUSTID})$},
  whose data are mapped to the ontology terms by means of the following mapping
  $\M$:
  {\footnotesize 
   \[
     \begin{array}{@{}l}
       m_{\msf{P}}{:}\, \text{\scriptsize\texttt{SELECT\;ID\;AS\;X\;FROM\;ENT\;WHERE\;ENT.TYPE='P'}}
       \rightsquigarrow \person(\texttt{X})\\
       m_{\msf{N}}{:}\, \text{\scriptsize\texttt{SELECT\;NUM\;AS\;X,\,CUSTID\;AS\;Y\;FROM PROD}}
       {\,\rightsquigarrow\,} \name(\texttt{X},\texttt{Y})\\
       m_{\msf{C}}{:}\, \text{\scriptsize\texttt{SELECT\;NUM\;AS\;X\;FROM\;PROD\;P\;WHERE\;P.TYPE='B'}}
       \rightsquigarrow \caccount(\texttt{X})
     \end{array}
   \]}%
  We will work with the corresponding high-level mapping $\M^b$ consisting of
  the assertions:
  {\small 
   \[
     \begin{array}[b]{@{}l@{~}r@{~}l}
       h_{\msf{P}}: & \{x &\mid \personview(x)\} \rightsquigarrow  \person(x) \\
       h_{\msf{N}}: & \{x,y &\mid \nameview(x,y)\} \rightsquigarrow  \name(x,y) \\
       h_{\msf{C}}: & \{x &\mid \caccountview(x)\} \rightsquigarrow  \caccount(x)
     \end{array}
   \]}%
  Now, consider the OBDA specification $\P^b=\tup{\T^b,\M^b,\S^b}$.  The
  $\rewcomp$ algorithm invoked on $(\T^b,\M^b)$ produces:
  \begin{itemize}
  \item The intermediate TBoxes $\T^b_1$ and $\T^b_2$ coinciding with
    $\T^b$, and $\T^b_3$ extending $\T^b$ with $A_{\caccount \AND A_1}
    \EQU \caccount \AND A_1$.
  \item The \mbox{ET-mapping $\M^b_c {\,=\,} \etm_{\T^b_3}(\M^b)$, which extends
      $\M^b$} with the assertions \mbox{\small $\{x\mid \nameview(x,y),
      \personview(y)\} {\,\rightsquigarrow\,} A_1\!(x)$}, \mbox{\small $\{x\mid
      \caccountview(x), \nameview(x,y), \personview(y)\} \rightsquigarrow
      \saccount(x)$}, and \mbox{\small $\{x\mid \caccountview(x),
      \nameview(x,y), \personview(y)\} \rightsquigarrow A_{\caccount \AND
        A_1}(x)$}.
  \end{itemize}
  The algorithm returns the \dlliter TBox $\T^b_r=\{A_{\caccount \AND A_1}
  \ISA \caccount$, $A_{\caccount \AND A_1} \ISA A_1$, $A_{\caccount \AND A_1}
  \ISA \saccount\}$ and the mapping $\M^b_c$.  It is possible to show that
  $\P^b_{\dlliter}=\tup{\T^b_r,\M^b_c,\S^b}$ is a CQ-rewriting of $\P^b$ into
  \dlliter.
  %
  \qedfull
\end{bigexample}

\begin{example}\label{ex:dllitertbox}
  The procedure terminates by returning the mapping $\M^b_c=\M^b_4$
  and the \dlliter TBox $\T^b_r$ obtained by removing from $\T^b_4$
  all non-\dlliter assertions: $\T^b_r=\{A_{\caccount \AND A_3} \ISA
  \saccount\}$
\qedfull
\end{example}

\begin{example}\label{ex:queryex}
  \nb{do we need to modify the procedure? to plug-in $M_L$?}Consider the OBDA
  specification $\P^b$ given in Example~\ref{ex:obda}.  Let
  $\P^b_{\dlliter}=\tup{\T^b_r,\M^b_c,\S}$, where $\T^b_r$ and $\M^b_c$ are as
  in Example~\ref{ex:dllitertbox}. It is possible to show that the rewritings
  (represented here only in terms of high-level mapping) with respect to both
  $\P^b$ and $\P^b_{\dlliter}$ of the query $q(x)=\saccount(x)$ coincide and
  are given by the query
  $q1_h(x) = \caccountview(x),\ \nameview(x,y),\ \personview(y)$.
  \qedfull
  %
\end{example}

The TBox $\T_3$ obtained as an intermediate result in Step~3 of
$\rewcomp(\T,\M)$, is a model-conservative extension of $\T$ that is
tailored towards capturing in \dlliter the answers to tree-shaped CQs.
This is obtained by introducing in Step~2 sufficiently new role
names, and in Step~3 new concept names, so as to capture entailed
axioms that 
generate the tree-shaped parts of models.
On the other hand, the ET-mapping $\M_c=\compile(\T,\M)$ is such that it
generates from a database instance a virtual ABox $\A^v$ that is complete with
respect to all ABox facts that might be involved in the generation of
the tree-shaped parts of models of $\T_r$ and $\A^v$.
This allows us to prove the main result of this section.

\begin{theorem}
  \label{thm:obda-rew-insep}
  Let $\tup{\T,\M,\S}$ be an OBDA specification such that $\T$ is a \halchiq
  TBox, and let $\tup{\T_r,\M_c}=\rewcomp(\T,\M)$. Then $\tup{\T,\M,\S}$ and
  $\tup{\T_r,\M_c,\S}$ are $\Sigma$-CQ inseparable, for
  $\Sigma=\sig(\T)$.
\end{theorem}

Clearly, $\tup{\T_r,\M_c,\S}$ is a candidate for being a CQ-rew\-ri\-ting of
$\tup{\T,\M,\S}$ into \dlliter. However, since $\M_c$ might be an infinite set,
$\tup{\T_r,\M_c,\S}$ might not be an OBDA specification and hence might not be
effectively usable for query answering.  Next we address this issue, and show
that in some cases we obtain proper CQ-rewritings, while in others we have to
resort to approximations.

\section{Approximating OBDA Specifications}
\label{sec:approximation}

To obtain from an ET-mapping a proper mapping, we exploit the notion of
predicate boundedness in Datalog, and use a bound on the depth of Datalog
expansion trees.

An IDB predicate $N$ is said to be \emph{bounded} in a Datalog program $\Pi$,
if there exists a constant $k$ depending only on $\Pi$ such that, for every
database $\D$, we have $N_\Pi^k(\D) = N_\Pi^\infty(\D)$ \cite{CGKV88}. 
If $N$ is bounded in $\Pi$, then there exists an equivalent Datalog program
$\Pi'$ such that $\Phi_{\Pi'}(N)$ is \emph{finite}, and thus represents a
finite union of CQ$^{\neq}$s.
It is well known that predicate boundedness for Datalog is undecidable in
general \cite{GMSV87}.
%
%
%
We say that $\Omega$ is a \emph{boundedness oracle} if for a Datalog program
$\Pi$ and a predicate $N$ it returns one of the three answers: $N$ is bounded
in $\Pi$, $N$ is not bounded in $\Pi$, or unknown.  When $N$ is bounded,
$\Omega$ returns also a \emph{finite} union of CQ$^{\neq}$s, denoted
$\Omega_\Pi(N)$, defining $N$.
Given a constant $k$, $\Phi_\Pi^k(N)$ denotes the set of trees (and the
corresponding union of CQ$^{\neq}$s) in $\Phi_\Pi(N)$ of depth at most $k$,
hence $\Phi_\Pi^k(N)$ is always finite.

We introduce a \emph{cutting operator} $\cut_k^\Omega$, which is parametric
with respect to the cutting depth $k>0$ and the boundedness oracle $\Omega$,
which, when applied to a predicate $N$ and a Datalog program $\Pi$, returns a
finite union of CQ$^{\neq}$s as follows:\\[1mm]
$\cut_k^{\Omega} \big(N, \Pi\big) =
\begin{cases}
  \Omega_\Pi(N), &\text{if $N$ is bounded in $\Pi$ w.r.t.~}\Omega\\
  \Phi_\Pi^k(N), &\text{otherwise}.
\end{cases}$\\[1mm]
We apply cutting also to ET-mappings: given an ET-mapping $\etm_\T(\M)$, the
\emph{mapping} $\cut_k^\Omega(\etm_\T(\M))$ is the (finite) set of mapping assertions
$\varphi^N(\vec{x}) \rightsquigarrow N(\vec{x})$ s.t.\ $N$ is a concept or
role name in~$\T$, and $\varphi^N\in\cut_k^\Omega(N, \Pi_{\T,\M})$ is
DB-defined.

\smallskip%
The following theorem provides a sufficient condition for CQ-rewritability into
\dlliter in terms of the well-known notion of first-order (FO)-rewritability,
which we recall here: a query $q$ is \emph{FO-rewritable} with respect to a
TBox $\T$, if there exists a FO query $q'$ such that
$\cert(q,\tup{\T,\A})=\ans(q',\A)$, for every ABox $\A$ over $\sig(\T)$ (viewed
as a database).
It uses the fact that if an AQ is FO-rewritable with respect to a \halchiq TBox
$\T$, then it is actually rewritable into a union of CQ$^{\neq}$s, and the fact
that if $\T$ is FO-rewritable for AQs (\ie, every AQ is FO-rewritable with
respect to $\T$), then each concept and role name is bounded in $\Pi_\T$
\cite{LuWo11,BiLW13}.%
%
%

\begin{theorem}
  \label{thm:cq-rewritability}
  Let $\tup{\T,\M,\S}$ be an OBDA specification such that $\T$ is a \halchiq
  TBox. Further, let $\T_r=\rew(\T)$ and $\M'=\cut^\Omega_k(\compile(\T,\M))$,
  for a boun\-ded\-ness oracle $\Omega$ and some $k>0$.
  If $\T$ is FO-rewritable for AQs, then $\tup{\T,\M,\S}$ is
  \emph{CQ-rewritable} into \dlliter, and $\tup{\T_r,\M',\S}$ is its
  \emph{CQ-rewriting}.  Otherwise, $\tup{\T_r,\M',\S}$ is a \emph{sound
   CQ-approximation} of $\tup{\T,\M,\S}$ in \dlliter.
\end{theorem}

The above result provides us with decidable conditions for rewritability
of OBDA specifications in several significant cases.  It is shown by
\citeasnoun{BiLW13} and \citeasnoun{LuWo11} that FO-rewritability of AQs
relative to \hshi-TBoxes, \halcf-TBoxes, and \halcif-TBoxes of depth two is
decidable.  In fact, these FO-rewritability algorithms provide us with a
boundedness oracle $\Omega$: for each concept and role name $N$ in $\T$, they
return a FO-rewriting of the AQ $N(\vec{x})$ that combined with the mapping
$\M$ results in $\Omega_{\Pi_{\T,\M}}(N)$.
%
%

Unfortunately, a complete characterization of CQ-rewritability into \dlliter is
not possible if arbitrary FO-queries are allowed in the (low-level) mapping.
\begin{theorem}
  \label{thm:rewritability-fo-mapping-undecidable}
  The problem of checking whether an OBDA specification with an \el ontology
  and FO source queries in the mapping is CQ-rewritable into \dlliter is
  undecidable.
\end{theorem}

However, if we admit only unions of CQs in the (low-level) mapping, we can
fully characterize CQ-rewritability.

\begin{theorem}
  \label{thm:rewritability-ucq-mapping-decidable}
  The problem of checking whether an OBDA specification with a
  Horn-$\mathcal{ALCHI}$ ontology of depth one and unions of CQs as source
  queries in the mapping is CQ-rewritable into \dlliter is decidable.
\end{theorem}

\section{Implementation and Experiments}

\label{sec:experiments}

To demonstrate the feasibility of our OBDA specification rewriting technique,
we have implemented a prototype system called
\ontoprox\footnote{\scriptsize\url{https://github.com/ontop/ontoprox/}} and
evaluated it over synthetic and real OBDA instances.  Our system relies on the
OBDA reasoner \ontop\footnote{\scriptsize\url{http://ontop.inf.unibz.it/}} and
the complete \mbox{\hshiq} CQ-answering system
\clipper\footnote{\scriptsize\url{http://www.kr.tuwien.ac.at/research/systems/clipper/}},
used as Java libraries.
\ontoprox also relies on a standard Prolog engine
(\system{swi-Prolog}\footnote{\scriptsize\url{http://www.swi-prolog.org/}}) and
on an \owltwo reasoner
(\system{HermiT}\footnote{\scriptsize\url{http://hermit-reasoner.com/}}).

Essentially, \ontoprox implements the rewriting and compiling procedure
described in Figure~\ref{fig:rewcomp}, but instead of computing the (possibly
infinite) ET-mapping $\compile(\T,\M)$, it computes its finite part
$\cut_k(\compile(\T\!,\M))$.  So, it gets as input an \owltwo OBDA specification
$\tup{\T_{\text{OWL2}},\M,\S}$ and a positive \mbox{integer~$k$}, and produces a
\dlliter OBDA specification that can be used with any OBDA system.
Below we describe some of the implementation details:
\begin{enumerate}[(1)]
\item $\T_{\text{OWL2}}$ is first approximated to the \hshiq TBox $\T$ by
  dropping the axioms outside this fragment.
\item $\T$ is translated into a (possibly recursive) Datalog program $\Pi$ and
  saturated with all CIs of the form $\bigsqcap A_i \ISA \SOME{R}{(\bigsqcap
    A_j')}$, using functionalities provided by \clipper.
\item The expansions $\cut_k(\Phi_\Pi(X))$ are computed by an auxiliary Prolog
  program using Prolog meta-programming.
\item To produce actual mappings that can be used by an OBDA reasoner, the
  views in the high-level mapping $\cut_k(\compile(\T,\M))$ are replaced with
  their original SQL definitions using functionalities of \ontop.
\item The \dlliter closure is computed by relying on the \owltwo reasoner for
  \hshiq TBox classification.
\end{enumerate}

%
%

\setlength{\tabcolsep}{3pt}

\begin{table*}[t]
  \centering
  \caption{Query evaluation with respect to 5 setups (number of answers /
   running time in seconds)}
  \begin{tabular}{llr@{\quad}r@{\quad}r@{\quad}r@{\quad}r}
    \toprule &  & \multicolumn{1}{c}{\ontop} & \multicolumn{1}{c}{LSA} & \multicolumn{1}{c}{GSA} & \multicolumn{1}{c}{\ontoprox} & \multicolumn{1}{c}{\clipper}\\
    \midrule
    UOBM & $Q_1^u$ & 14,129~/~0.08 & 14,197~/~0.11 & 14,197~/~0.43& 14,197~/~0.42 & 14,197~/~21.4\\
    & $Q_2^u$ & 1,105~/~0.09 & 2,170~/~0.15 & 2,170~/~0.42& 2,170~/~0.44 & 2,170~/~21.3\\
    & $Q_3^u$ & 235~/~0.20 & 235~/~0.24 & 235~/~0.88 &247~/~0.83 & 247~/~19.6\\
    & $Q_4^u$ & 19~/~0.13 & 19~/~0.15 & 19~/~0.43 & 38~/~0.52 & 38~/~21.4\\
    \midrule
    Telecom &   $Q_1^t$ & 0~/~2.91         & 0~/~0.72             & 0~/~1.91        & 82,455~/~5.21        & \multicolumn{1}{c}{N/A}\\
    &           $Q_2^t$ & 0~/~0.72        & 0~/~0.21            & 0~/~0.67        &16,487~/~\,198         & \multicolumn{1}{c}{N/A}\\
    &           $Q_3^t$ & 5,201,363~/~\,\,128 & 5,201,363~/~\,\,105     & 5,201,363~/~\,\,538 & 5,260,346~/~\,437     & \multicolumn{1}{c}{N/A}\\
    \bottomrule
  \end{tabular}
  \label{tab:query-evaluation}
\end{table*}



%
%

\begin{table}[t]
  \centering
  \caption{\ontoprox pre-computation time and output size}
  \begin{tabular}{l@{~}c@{~~}c@{~~}c@{~~}c@{~~}c}
    \toprule
    & UOBM & Telecom \\
    \midrule
    Time (s) & 8.47 & 8.72\\
    Number of mapping assertions & 441 & 907\\
    Number of TBox axioms & 294 & 620\\
    Number of new concepts & 26 & 60\\
    Number of new roles & 30 & 7\\ 
    \bottomrule
  \end{tabular}

  \label{tab:obql-time-size}
\end{table}




For the experiments, we have considered two scenarios:

\smallskip
\noindent
\textbf{UOBM.~}
The university ontology benchmark (UOBM) \cite{MYQX*06} comes with
  a $\mathcal{SHOIN}$ ontology (with $69$ concepts, $35$ roles, $9$ attributes,
  and $204$ TBox axioms), and an ABox generator.  We have
  designed a
  database schema for the generated ABox,
  converted the ABox to a 10MB database instance for the schema,
  and manually created the mapping, consisting of $96$
  assertions\footnote{\scriptsize\url{https://github.com/ontop/ontop-examples/tree/master/aaai-2016-ontoprox/uobm}}.

  Among others, we have considered the following queries:

  \begin{enumerate}[\hspace{1cm}]\footnotesize
  \item[$Q_1^u$:]
    {\tt\begin{tabular}[t]{@{}l}
      SELECT DISTINCT ?X WHERE\\
      ~ \{ ?X a ub:Person . \}
    \end{tabular}}
  \item[$Q_2^u$:]
    {\tt\begin{tabular}[t]{@{}l}
      SELECT DISTINCT ?X WHERE\\
      ~ \{ ?X a ub:Employee . \}
    \end{tabular}}
  \item[$Q_3^u$:]
    {\tt\begin{tabular}[t]{@{}l}
      SELECT DISTINCT ?X ?Y WHERE\\
      ~ \{ \begin{tabular}[t]{@{}l}
        ?X rdf:type ub:ResearchGroup .\\
        ?X ub:subOrganizationOf ?Y . \}
      \end{tabular}
    \end{tabular}}
  \item[$Q_4^u$:]
    {\tt\begin{tabular}[t]{@{}l}
      SELECT DISTINCT ?X ?Y ?Z WHERE\\
      ~ \{ \begin{tabular}[t]{@{}l}
        ?X rdf:type ub:Chair .\\
        ?X ub:worksFor ?Y .\\
        ?Y rdf:type ub:Department . \\
        ?Y ub:subOrganizationOf ?Z . \}
      \end{tabular}
    \end{tabular}}
  \end{enumerate}

\smallskip
\noindent
\textbf{Telecom benchmark.~}  The telecommunications ontology models a portion of
  the network of a leading telecommunications company, namely the portion
  connecting subscribers to the operating centers of their service providers.
  The current specification consists of an \owltwo ontology with $152$
  concepts, $53$ roles, $73$ attributes, $458$ TBox axioms, and of a mapping
  with $264$ mapping assertions. The database instance contains 32GB of
  real-world data.

  In the following, we only provide a description of some of the queries
  because the telecommunications ontology itself is bound by a confidentiality
  agreement.
\begin{itemize} \itemsep 0cm
\item Query $Q_1^t$ asks, for each cable in the telecommunications network, the
  single segments of which the cable is composed, and the network line (between
  two devices) that the cable covers. For each cable, it also
  returns its bandwidth and its status (functioning, non-functioning, etc.).
\item Query $Q_2^t$ asks for each path in the network that runs on fiber-optic
  cable, to return the specific device from which the path originates, and also
  requires to provide the number of different channels available in
  the path.
\item Query $Q_3^t$ asks, for each cable in the telecommunications network, the
  port to which the cable is attached, the slot on the device in which the port
  is installed, and, for each such slot, its status and its type. For each
  cable, it also returns its status.
\end{itemize}

For each OBDA instance $\tup{\tup{\T,\M,\S},\D}$, we have evaluated the number
of query answers and the query answering time with respect to five different
setups:
\begin{enumerate}[(1)] \itemsep 0cm
\item The default behavior of \ontop v1.15, which simply ignores all
  non-\dlliter axioms in~$\T$, \ie, using $\tup{\T^1,\M,\S}$ where $\T^1$ are
  all the \dlliter axioms in~$\T$.
\item The local semantic approximation (LSA) of $\T$ in \dlliter, \ie, using
  $\tup{\T^2,\M,\S}$ where $\T^2$ is obtained as the union, for each axiom
  $\alpha \in \T$, of the set of \dlliter axioms $\Gamma(\alpha)$ entailed by
  $\alpha$~\cite{CMRSS14}. 
\item The global semantic approximation (GSA) of $\T$ in \dlliter, \ie, using
  $\tup{\T^3,\M,\S}$ where $\T^3$ is the \dlliter closure of
  $\T$~\cite{PaTh07}.
\item Result of \ontoprox, $\tup{\rew(\T),\cut_5(\compile(\T\!,\M)), \S}$.
\item \clipper over the materialization of the virtual ABox.
\end{enumerate}
In Table~\ref{tab:query-evaluation}, we present details of the evaluation for
some of the queries for which we obtained significant results.  In
Table~\ref{tab:obql-time-size}, we provide statistics about the \ontoprox
pre-computations.  The performed evaluation led to the following findings:
\begin{itemize} \itemsep 0cm
\item For the considered set of queries LSA and GSA produce the same answers.
\item Compared to the default \ontop behavior, LSA/GSA produces more answers for
  $2$ queries out of $4$ for UOBM.
\item \ontoprox produces more answers than LSA/GSA for $2$ queries out of $4$
  for UOBM, and for all Telecom queries.  In particular, note that for $Q_1^t$
  and $Q_2^t$, LSA and GSA returned no answers at all.
\item For UOBM, \ontoprox answers are complete, as confirmed by the comparison
  with the results provided by \clipper.  We cannot determine completeness for
  the Telecom queries, because the Telecom database was too large and its
  materialization in an ABox was not feasible.
\item Query answering of \ontoprox is \texttildelow$3$--$5$ times slower than
  \ontop, when the result sets are of comparable size (note that for $Q_2^t$
  the result set is significantly larger).
\item The size of the new \dlliter OBDA specifications 
  is comparable with that of the original specifications.
\end{itemize}

\section{Conclusions}
\label{sec:conclusions}


We proposed a novel framework for rewriting and approximation of OBDA
specifications in an expressive ontology language to specifications in a weaker
language, in which the core idea is to exploit the mapping layer to encode part
of the semantics of the original OBDA specification, and we developed
techniques for \dlliter as the target language.

We plan to continue our work along the following directions:
\begin{inparaenum}[\it (i)]
\item extend our technique to \hshiq, and, more generally, to Datalog
  rewritable TBoxes \cite{CMSH13};
\item deepen our understanding of the computational complexity of deciding
  CQ-rewritability of OBDA specifications into \dlliter;
\item extend our technique to SPARQL queries under different OWL entailment
  regimes \cite{KRRXZ-iswc14};
\item carry out more extensive experiments, considering queries that contain
  existentially quantified variables.  This will allow us to verify the
  effectiveness of \rewcomp, which was designed specifically to deal with
  existentially implied objects.
\end{inparaenum}


\vspace{2mm}
\noindent\textbf{Acknowledgement.} This paper is supported by the EU under
the large-scale integrating project (IP) Optique (\emph{Scalable End-user
 Access to Big Data}), grant agreement n.~FP7-318338.  We thank
Martin Rezk for insightful discussions, and Benjamin Cogrel and Elem G{\"u}zel
for help with the experimentation.

\bibliographystyle{aaai}
\bibliography{main-bib}

\ifextendedversion
\clearpage
\appendix

\newenvironment{theoremnum}[1]{\smallskip\noindent\textbf{Theorem~#1.}
  \hspace*{0.3em}\em}{\par\smallskip}
\newenvironment{lemmanum}[1]{\smallskip\noindent\textbf{Lemma~#1.}
  \hspace*{0.3em}\em}{\par\smallskip}

\section{Appendix}

\subsection{Expansion of Datalog Programs}
\label{sec:datalog-expansion}
We recall here the notion of the expansion trees \cite{CGKV88}. 
Formally, an \emph{expansion tree} for a predicate $N$ in a Datalog program
$\Pi$ is a \emph{finite} tree $\varphi^N_{\Pi}$ satisfying the following
conditions:
\begin{compactitem}
\item Each node $x$ of $\varphi^N_{\Pi}$ is labeled by a pair of the form
  $(\alpha_x,\rho_x)$, where $\alpha_x$ is an IDB atom and $\rho_x$ is an
  instance of a rule of $\Pi$ such that the head of $\rho_x$ is $\alpha_x$.
  Moreover, the variables in the body of $\rho_x$ either occur in $\alpha_x$ or
  they do not occur in the label of any node above $x$ in the tree.
\item The IDB atom labeling the root of $\varphi^N_{\Pi}$ is an $N$-atom.
\item If $x$ is a node, where $\alpha_x = Y(\vec{t})$,
  $\rho_x=Y(\vec{t})\leftarrow Y_1(\vec{t}_1),\dots,Y_m(\vec{t}_m)$, and the
  IDB atoms in the body of $\rho_x$ are
  $Y_{i_1}(\vec{t}^{i_1}),\dots,Y_{i_\ell}(\vec{t}^{i_\ell})$, then $x$ has
  $\ell$ children, respectively labeled with the atoms
  $Y_{i_1}(\vec{t}^{i_1}),\ldots,Y_{i_\ell}(\vec{t}^{i_\ell})$.  In particular,
  if $\rho_x$ is an initialization rule (\ie, the body of $\rho_x$ does not
  contain an IDB predicate), then $x$ is a leaf.
\end{compactitem}

\subsection{Proofs of Section~\ref{sec:extended-t-mappings}}

In the following, for an OBDA specification $\P$ and a database instance $\D$,
if $\tup{\P,\D}$ has a model, we say that $\D$ is \emph{consistent with} $\P$.

\begin{lemmanum}{\ref{lem:t-m-program}}
  Let $\tup{\T, \M, \S}$ be an OBDA specification where $\T$ is Datalog
  rewritable.  Then, for every database instance $\D$ of $\S$, concept or role
  name $N$ of $\T$, and $\vec{a}$ in $\ind(\A_{\M,\D})$, we have that
  $\tup{\T,\A_{\M,\D}} \models N(\vec{a})$ iff
  $N(\vec{a}) \in N_{\Pi_{\T,\M}}^\infty(\D)$.
\end{lemmanum}
\begin{proof}
  We assume that the Datalog translation $\Pi_\T$ of $\T$ satisfies the
  following properties (see, \eg, Theorem~1 in \cite{HuMS05} and Proposition~2
  in \cite{EOSTX12}):
  \begin{itemize}
  \item[$\star$] for every ABox $\A$, $\tup{\T,\A}$ is consistent iff
    $\bot_{\Pi_{\T}}^\infty(\A)$ is empty, and if $\tup{\T,\A}$ is consistent,
    then for every concept or role name $N$ of $\T$ and $\vec{a}$ in
    $\ind(\A)$, $\tup{\T,\A} \models N(\vec{a})$ iff $N(\vec{a}) \in
    N_{\Pi_{\T}}^\infty(\A)$.
  \end{itemize}

  Recall that for an ABox $\A$ such that $\tup{\T,\A}$ is inconsistent,
  $\tup{\T,\A}$ entails all possible facts of the form $N(\vec{a})$ for a
  concept or role name $N$ of $\T$ and $\vec{a}$ in $\ind(\A_{\M,\D})$.  We
  prove that the translation $\Pi_\T$ containing the auxiliary rules involving
  $\adom$ and $\bot$ satisfies a stronger property:
  \begin{itemize}
  \item[$\star\star$] for every ABox $\A$, for every concept or role name $N$
    of $\T$ and $\vec{a}$ in $\ind(\A)$, $\tup{\T,\A} \models N(\vec{a})$ iff
    $N(\vec{a}) \in N_{\Pi_{\T}}^\infty(\A)$.
  \end{itemize}
  Indeed, it is easy to see that considering that
  \begin{inparaenum}[\it (i)]
  \item $\adom$ contains precisely $\ind(\A)$, \ie, $\adom(a) \in
    {\adom}_{\Pi_{\T}}^\infty(\A)$ for each $a\in\ind(\A)$, and
  \item if $\bot$ is true in $\Pi_\T(\A)$, then for each concept or role name $N$
    of $\T$ and $\vec{a}$ in $\ind(\A_{\M,\D})$, we have that $N(\vec{a}) \in
    N_{\Pi_{\T}}^\infty(\A)$.
  \end{inparaenum}

  As $\Pi_{\T,\M} = \Pi_\T \cup \Pi_\M$, the statement of the lemma follows
  directly from the property $\star\star$ of $\Pi_\T$ and the fact that the
  rules in $\Pi_\M$ connect two disjoint vocabularies.
\end{proof}

\begin{lemma}
  \label{lem:et-mapping}
  Let $\tup{\T, \M, \S}$ be an OBDA specification where $\T$ is Datalog
  rewritable, 
  and $\M' = \etm_\T(\M)$. Then for every database instance $\D$ of $\S$, we
  have that $\A_{\M', \D}$ is exactly the set of all facts entailed by
  $\tup{\T\!,\,\A_{\M,\D}}$, \ie, assertions of the form $A(a)$, $P(a,b)$ for
  $a,b \in \ind(\A_{\M,\D})$, $A,P \in \sig(\T)$.
\end{lemma}
\begin{proof}
  Let $\D$ be a database instance of $\S$, $a \in \ind(\A_{\M,\D})$, $A$ a
  concept name in $\sig(\T)$. Then
  \begin{itemize}
  \item $\tup{\T\!,\,\A_{\M,\D}} \models A(a)$
    iff (by Lemma~\ref{lem:t-m-program})
  \item $A(a) \in A_{\Pi_{\T,\M}}^\infty(\D)$ iff (by the properties of expansions)
  \item there exists DB-defined $\varphi \in \Phi_{\Pi_{\T,\M}}(A)$ such that
    $a \in \ans(\varphi,\D)$.
  \end{itemize}

  Let $\tup{\T\!,\,\A_{\M,\D}} \models A(a)$. By construction of $\M'$, the
  mapping assertion $\Phi_{\Pi_{\T,\M}}(A) \leadsto A(x)$ is in $\M'$. We
  conclude that $A(a) \in \A_{\M', \D}$.
  Now, assume that $A(a) \in \A_{\M', \D}$. It follows that in $\M'$ there is a
  mapping assertion $\Phi_{\Pi_{\T,\M}}(A) \leadsto A(x)$ and $a \in
  \ans(\varphi, \D)$ for some $\varphi \in \Phi_{\Pi_{\T,\M}}(A)$. We conclude
  that $\tup{\T\!,\,\A_{\M,\D}} \models A(a)$.

  The proof for role assertions is analogous.
\end{proof}

\subsection{Proof of Theorem~\ref{thm:obda-rew-insep}}

We start by showing a sufficient condition for $\Sigma$-CQ inseparability of
OBDA specifications.

A \emph{homomorphism} between two interpretations is a mapping between their
domains that preserves constants and relations.  A model of an ontology $\O$
that can be homomorphically embedded in every model of $\O$ is called a
\emph{canonical model} of $\O$, from now on denoted $\C_\O$.  The notion of
canonical model is important in the context of CQ answering, since answers to
CQs are preserved under homomorphisms, \ie, if $q(\vec{a})$ holds in an
interpretation $\I$, and there is a homomorphism from $\I$ to an interpretation
$\I'$, then $q(\vec{a})$ holds in $\I'$ \cite{ChMe77}.  It follows that certain
answers can be characterized as the answers over the canonical model.

\begin{lemma}\label{lem:query-kb-vs-uni}
  Let $\O=\tup{\T,\A}$ be a consistent ontology that has a canonical model
  $\C_{\O}$, $q(\vec x)$ a CQ, and $\vec{a}$ a tuple from $\ind(\A)$.  Then
  $\vec{a} \in \cert(q, \O)$ iff $\C_{\O}$ satisfies $q(\vec{a})$.
\end{lemma}
It is well known that Horn variants of DLs have the \emph{canonical model
 property} \cite{EOSTX12,BKRWZ14}, \ie, every satisfiable ontology $\O$ admits
a (possibly infinite) canonical model.

A $\Sigma$-homomorphism is a mapping that preserves constants and relations in
$\Sigma$.  Given two interpretations $\I$ and $\J$, we say that $\I$ is
($\Sigma$-)homomorphically embeddable into $\J$ if there exists a
($\Sigma$-)homomorphism from $\I$ to $\J$. Moreover, $\I$ and $\J$ are
($\Sigma$-)homomorphically equivalent if they are \mbox{($\Sigma$-)}homomorphically
embeddable into each other. The following characterization can be easily
derived from Lemma~\ref{lem:query-kb-vs-uni}, the property of canonical models,
and the definitions of certain answers.
\begin{lemma}
  \label{lem:inseparability-characterization}
  Let $\Sigma$ be a signature, and $\P_1=\tup{\T_1, \M_1, \S}$ and
  $\P_2=\tup{\T_2,\M_2,\S}$ two OBDA specifications. 
  Assume that for every ABox $\A$, both $\tup{\T_1,\A}$ and $\tup{\T_2,\A}$
  admit a canonical model.  If
  \begin{itemize}
  \item for every database instance $\D$ of $\S$ that is consistent with both
    $\P_1$ and $\P_2$, we have that $\C_{\tup{\T_1,\A_{\M_1,\D}}}$ and
    $\C_{\tup{\T_2,\A_{\M_2,\D}}}$ are $\Sigma$-homomorphically equivalent, and
  \item for every database instance $\D$ of $\S$ that is inconsistent with
    $\P_1$ or $\P_2$, we have that $\tup{\T_1,\A_{\M_1,\D}}$ and
    $\tup{\T_2,\A_{\M_2,\D}}$ entail the same ABox facts over $\Sigma$,
  \end{itemize}
then $\P_1$ and $\P_2$ are $\Sigma$-CQ inseparable.
\end{lemma}

\smallskip%
Now we show several properties of the intermediate and final TBoxes obtained
during the $\rewcomp(\T,\M)$ procedure. First, we show that, for ABoxes
containing a single assertion, the \dlliter TBox $\T_r$ generates a canonical
model equivalent to the canonical model generated by the intermediate \halchiq
TBox $\T_3$.
\begin{lemma}
\label{lem:rew-anon-part}
Let $\T$ be a \halchiq TBox, $\T_3$ the TBox obtained in step~3 of
$\rewcomp(\T,\M)$, and $\T_r=\rew(\T)$.  Also, for $A$ a concept name in
$\sig(\T_3)$ and $a\in\indivnames$, let $\A=\{A(a)\}$ be an ABox such that
$\tup{\T_3,\A}$ is consistent.  Then $\C_{\tup{\T_3,\A}}$ is homomorphically
equivalent to $\C_{\tup{\T_r,\A}}$.
\end{lemma}
\begin{proof}
  In this proof, we call an element $\sigma'$ a successor of $\sigma$ in a
  canonical model $\C$ of $\tup{\T,\A}$, for $\sigma, \sigma' \in \dom[\C]$, if
  $\sigma'$ is added (\ie, generated) to satisfy an existential assertion
  $\alpha$ of the form $\bigsqcap A_i \ISA \SOME{(\bigsqcap Q_j)}{(\bigsqcap
    A'_k)}$ such that $\sigma \in \Int[\C]{A_i}$. Here we assume a construction
  of the canonical model that is ``minimal'' and uniquely defined: a new
  successor $\sigma'$ of $\sigma$ is generated only if there is no other
  element $\delta$ of $\dom[\C]$ that can be used to satisfy $\alpha$, in this
  case we employ the following naming convention: $\sigma'$ is a path of the
  form $\sigma \cdot w_{(\{Q_j\}, \{A'_k\})}$. In particular, if there is a
  role $P$ in $\T$ such that it only appears on the left-hand side of role
  inclusions, and $\T \models \{\bigsqcap A_i \ISA \SOMET{P}, P \ISA Q_j,
  \SOMET{P^-} \ISA A'_k\}$, then $\sigma$ has to have a successor $\delta$
  introduced to satisfy the assertion $\bigsqcap A_i \ISA \SOMET{P}$, so
  $\delta$ can be used to satisfy $\alpha$ and no new successor $\sigma'$ is
  generated.
  Also, if the predecessor $\delta'$ of $\sigma$ is such that $\delta' \in
  \Int[\C]{(A'_k)}$ and $(\sigma, \delta') \in \Int[\C]{Q_j}$, then no new
  successor is introduced.
  Here, we call each fresh role name $P_{\mathit{new}}$ introduced by the
  normalization procedure $\normexists$ a \emph{generating \dlliter-role}
  precisely because it satisfies the above properties. Note that for a
  generating \dlliter-role $P$, the concept $\SOMET{P^-}$ has no non-empty
  sub-concepts and $P$ does not appear in constructs $\SOME{P}{C}$ where $C$ is
  a concept distinct from $\top$. Therefore, the element introduced as a
  $P$-successor of $\sigma$ can be simply named $\sigma \cdot w_{P}$ (or
  $\sigma \cdot v_P$ to distinguish between two canonical models).

  \smallskip%
  Observe that $\sig(\T_3) = \sig(\T_r)$.
  Let $A$ be a satisfiable concept name in $\sig(\T_3)$, and
  $\A=\{A(a)\}$. Denote by $\C_1$ the canonical model of $\tup{\T_3,\A}$, and
  by $\C_2$ the canonical model of $\tup{\T_r, \A}$.

  \medskip%
  First, for each element in $\dom[\C_1]$ distinct from $a$, we prove that it
  is of the form $a w_1 \cdots w_n$, where each $w_i = w_P$ for some
  generating \dlliter-role $P$.
  Suppose that $\sigma \in \dom[\C_1]$, $\sigma \in \Int[\C_1]{B_i}$, $1 \leq i
  \leq k$, and $\T_3 \models \alpha$ where $\alpha = B_1 \AND \cdots \AND B_k
  \ISA \SOME{Q}{(A_1 \AND \cdots \AND A_m)}$. By induction on the length of
  $\sigma$, we find an element $\delta \in \dom[\C_1]$ such that $\delta \in
  \Int[\C_1]{A_i}$, $1 \leq i \leq m$, $(\sigma,\delta) \in \Int[\C_1]{Q}$, and
  $\delta$ is of the desired form.
  Note that without loss of generality we may assume that none of $B_1, \dots,
  B_k$ and none of $A_1, \dots, A_m$ is a fresh concept name introduced by
  $\normand$ as we can substitute each such name with its
  definition. Therefore, $\{B_1, \dots, B_k, A_1, \dots, A_m\} \subseteq
  \sig(\T)$.
  Consider the following cases:

  \noindent (a) If $Q \in \sig(\T)$, then by step~1 $\alpha
  \in \T_1$, and by $\normexists$ in step~2, $\T_2$ (hence, $\T_3$) contains
  axioms $B_1 \AND \cdots \AND B_k \ISA \SOMET{P}$, $P \ISA Q$, $\SOMET{P^-}
  \ISA A_i$, $1 \leq i \leq m$, for a fresh role name $P$.  If $\sigma=a$, then
  we set $\delta = a w_{P} \in \dom[\C_1]$, for which it holds that that $a w_P
  \in \Int[\C_1]{A_i}$, $1 \leq i \leq m$, and $(a,aw_P) \in \Int[\C_1]{Q}$. If
  $\sigma = \sigma' w_S$ for some generating \dlliter-role $S$, and it is not
  the case that $\T_3 \models S \ISA Q^-$ and $\sigma' \in \Int[\C_1]{A_i}$,
  then we set $\delta = \sigma w_{P}$, for which we have that $\sigma w_P \in
  \Int[\C_1]{A_i}$, $1 \leq i \leq m$, and $(\sigma, \sigma w_P) \in
  \Int[\C_1]{Q}$. Otherwise we set $\delta = \sigma'$, which is of the desired
  form.

  \noindent (b) If $Q \notin \sig(\T)$, then $Q$ is a fresh generating \dlliter
  role introduced by $\normexists$. It means that there exists a CI $\bigsqcap
  A_j' \ISA \SOME{S}{(\bigsqcap A_l'')}$ in $\T_1$ such that $\T_2$ contains
  axioms $\bigsqcap A_j' \ISA \SOMET{Q}$, $Q \ISA S$, and $\top \ISA
  \ALL{Q}{A_l''}$.  Since $Q$ occurs only in these axioms, it must be the case
  that $\T_3 \models B_1 \AND \cdots \AND B_k \ISA \bigsqcap A_j'$ and $\T_3
  \models \bigsqcap A_l'' \ISA A_1 \AND \cdots \AND A_m$. Because of the
  former, we obtain that $\sigma \in \Int[C_1]{(A_j')}$, and similarly to (a)
  that there is an element $\delta\in \dom[\C_1]$ such that $\delta \in
  \Int[\C_1]{(A_l'')}$ and $(\sigma, \delta) \in \Int[\C_1]{Q}$. Because of the
  latter, we also have that $\delta \in \Int[\C_1]{A_i}$, $1 \leq i \leq m$.

  \medskip%
  Second, we show that there exists a $\Sigma$-homomorphism from $\C_1$ to
  $\C_2$, by constructing one.
  Let $a \in \Int[\C_1]{B}$ for a basic concept $B$. Then it must be that $\T_3
  \models A \ISA B$. By construction, $\T_r \models A \ISA B$, and therefore $a
  \in \Int[\C_2]{B}$. So we can set $h(a) = a$.

  Let $\sigma \in \dom[\C_1]$ such that $h(\sigma)$ is set, $h(\sigma) =
  \delta$, and $\sigma w_P \in \dom[\C_1]$. Then $\sigma \in
  \Int[\C_1]{(\SOMET{P})}$, hence $\delta \in \Int[\C_2]{(\SOMET{P})}$. Since
  $P$ is a generating \dlliter-role in $\T_3$, and $\T_r$ is derived from
  $\T_3$, it follows that $P$ is a generating \dlliter-role in $\T_r$, hence
  there exists a successor $\delta v_P \in \dom[\C_2]$. By construction of
  $\T_r$, it follows that we can set $h(\sigma w_P) = \delta v_P$ (that is, the
  homomorphism conditions are satisfied).

  \medskip%
  Finally, as for a homomorphism from $\C_2$ to $\C_1$, its existence
  follows from the fact that $\T_3 \models \T_r$: $\C_2$ is
  homomorphically embeddable into each model $\I$ of $\tup{\T_r, \A}$
  and as $\T_3 \models \T_r$, $\C_1$ is a model of $\tup{\T_r,\A}$ as
  well.
\end{proof}

In general, however, $\T_r$ does not preserve ABox entailments for non-singleton
ABoxes. 
\begin{bigexample}
  \label{ex:complex-abox}
  Consider the following TBox $\T$, together with the computed $\T_3$ and
  $\T_r$:
  {\footnotesize
  \[
    \begin{array}{l@{~}c@{~}l}
      \T &=& \{B \AND C \ISA A,~ B\AND C\ISA\SOMET{P}\}\\
      \T_3 &=& \T\cup\{A_{B\AND C} \EQU B \AND C,~ A_{B\AND C}\ISA \SOMET{P}\}\\
      \T_r &=& \{A_{B\AND C} \ISA A, A_{B \AND C} \ISA B, A_{B \AND C} \ISA C,
      A_{B\AND C} \ISA \SOMET{P}\}
    \end{array}
  \]}%
  Then, for $\A=\{B(a),C(a)\}$ an ABox, $\tup{\T_3, \A} \models A(a)$, however
  $\tup{\T_r,\A}\not\models A(a)$.
  \qedfull
\end{bigexample}

Next, we extend Lemma~\ref{lem:rew-anon-part} towards arbitrary ABoxes. To do
so, we consider ABoxes that are closed with respect to $\T_3$, where an ABox
$\A$ is \emph{closed with respect to a TBox $\T$} if $\A=\mathsf{EABox}(\T,\A)$
where $\mathsf{EABox}(\T,\A)$ is the set of all membership assertions over
$\sig(\T)$ and $\ind(\A)$ entailed by $\tup{\T,\A}$.
We say that an ABox $\A$ is \emph{complete} (within $\rewcomp(\T,\M)$), if
it is closed with respect to $\T_3$.
The following result is a corollary of Lemma~\ref{lem:rew-anon-part}.
%
%
\begin{lemma}
  \label{lem:rew-h-complete-aboxes}
  Let $\T$ be a \halchiq TBox, $\T_3$ the TBox obtained in step~3 of
  $\rewcomp(\T,\M)$, and $\T_r=\rew(\T)$.  Then for each ABox $\A$ that is
  complete and such that $\tup{\T_3,\A}$ is consistent, we have that
  $\C_{\tup{\T_3,\A}}$ and $\C_{\tup{\T_r, \A}}$ are homomorphically
  equivalent.
\end{lemma}
\begin{proof}
  We use the same assumptions for the canonical model as in the proof
  of Lemma~\ref{lem:rew-anon-part}.
  Denote by $\C_1$ the canonical model of $\tup{\T_3,\A}$, and by $\C_2$ the
  canonical model of $\tup{\T_r, \A}$ for a complete ABox $\A$.  The existence
  of a homomorphism from $\C_2$ to $\C_1$ is straightforward. We show that
  there exists a homomorphism from $\C_1$ to $\C_2$.
  Let $a \in \ind(\A)$, it is sufficient to show that for each successor $a
  w_P$ of $a$ in $\C_1$, there is a successor $a v_P$ of $a$ in $\C_2$.  The
  rest of the proof follows from Lemma~\ref{lem:rew-anon-part}.

  Let $\sigma$ be a successor of $a$ in $\C_1$. It follows from the proof of
  Lemma~\ref{lem:rew-anon-part} that $\sigma$ is of the form $a w_{P}$ where
  $P$ is a generating \dlliter role and $\tup{\T_3, \A} \models \SOMET{P}(a)$
  (moreover, there exists no individual $b \in \ind(\A)$ such that $R(a,b) \in
  \A$ for each role $R$ with $\T_3 \models P \ISA R$ and $B(b) \in \A$ for each
  concept name $B$ with $\T_3 \models \SOMET{P^-} \ISA B$).

  We show that $\tup{\T_r, \A} \models \SOMET{P}(a)$.
  Recall that $\T_3$ contains all possible CIs with $\SOMET{P}$ on the
  right-hand side.  From $\tup{\T_3, \A} \models \SOMET{P}(a)$, it follows that
  there exists a CI $A_1 \AND \cdots \AND A_n \ISA \SOMET{P}$ in $\T_3$, $n
  \geq 1$, such that $A_i(a) \in \A$. If $n=1$, then $\T_r \models A_1 \ISA
  \SOMET{P}$, hence $\tup{\T_r, \A} \models \SOMET{P}(a)$. Assume that $n > 1$,
  then by step~3 $\T_3$ contains $A_{A_1 \AND \cdots \AND A_n} \EQU A_1 \AND
  \cdots \AND A_n$, therefore $\T_r$ contains $A_{A_1 \AND \cdots \AND A_n}
  \ISA \SOMET{P}$, and since $\A$ is closed with respect to $\T_3$, $\A$
  contains $A_{A_1 \AND \cdots \AND A_n}(a)$. Finally, we obtain that $\T_r
  \models \SOMET{P}(a)$.

  Now, as no individual $b$ can be used as a $P$-successor of $a$, in $\C_2$
  there is a successor $a v_P$ of $a$.
\end{proof}
The result above is significant, because the mapping $\M_c=\compile(\T,\M)$ is
such that it generates from a database instance a virtual ABox that is complete
(within $\compile(\T,\M)$).  Finally, combined with Lemmas~\ref{lem:et-mapping}
and~\ref{lem:rew-h-complete-aboxes}, we are ready to prove
Theorem~\ref{thm:obda-rew-insep}.

\begin{theoremnum}{\ref{thm:obda-rew-insep}}
  Let $\tup{\T,\M,\S}$ be an OBDA specification, $\Sigma=\sig(\T)$, and
  $\tup{\T_r,\M_c}=\rewcomp(\T,\M)$.  Then, for each database instance $\D$ of
  $\S$ and for each $\Sigma$-query $q$, we have that $\cert(q, \tup{\T,\M,\S},
  \D) = \cert(q, \tup{\T_r,\M_c,\S}, \D)$.
\end{theoremnum}
\begin{proof}
  \textbf{(a)} Let $\D$ be an instance of $\S$ inconsistent with
  $\tup{\T,\M,\S}$. By Lemma~\ref{lem:et-mapping} $\A_{\M_c,\D} =
  \mathsf{EABox}(\tup{\T_3, \A_{\M,\D}})$, and since $\D$ is inconsistent with
  $\tup{\T_3,\M,\S}$, we have that $\A_{\M_c,\D}$ contains all possible facts
  over $\sig(\T)$. The axioms in $\T_r$ do not add more facts.

  \textbf{(b)} We show that for each database instance $\D$ of $\S$ consistent
  with $\tup{\T,\M,\S}$, the canonical models $\C_{\tup{\T,\A_{\M,\D}}}$ and
  $\C_{\tup{\T_r,\A_{\M_c,\D}}}$ are $\Sigma$-homomorphically equivalent.

  \smallskip%
  (I) We observe that by Lemma~\ref{lem:et-mapping}, it follows that
  $\A_{\M_c,\D}$ is closed with respect to $\T_3$.

  \smallskip%
  (II) We show that $\C_{\tup{\T, \A}}$ is $\Sigma$-homomorphically
  equivalent to $\C_{\tup{\T_3, \A}}$, where $\A$ is an ABox and
  $\Sigma = \sig(\T)$. The interesting direction is the existence of a
  $\Sigma$-homomorphism from $\C_{\tup{\T_3, \A}}$ to $\C_{\tup{\T,
      \A}}$. Since $\T_3$ is a model-conservative extension of $\T$,
  $\C_{\tup{\T, \A}}$ can be extended without changing the
  interpretations of symbols in $\Sigma$ to a model $\I$ of
  $\tup{\T_3, \A}$. By definition of canonical model, there exists a
  homomorphism from $\C_{\tup{\T_3, \A}}$ to $\I$ and since $\I$
  agrees with $\C_{\tup{\T,\A}}$ on $\Sigma$, we obtain that there
  exists a $\Sigma$-homomorphism from $\C_{\tup{\T_3, \A}}$ to
  $\C_{\tup{\T, \A}}$.

  \smallskip%
  (III) We show that $\mathsf{EABox}_\Sigma(\T,\A_{\M_c,\D}) =
  \mathsf{EABox}_\Sigma(\T,\A_{\M,\D})$, where $\Sigma = \sig(\T)$ and
  $\mathsf{EABox}_\Sigma(\T,\A)$ is the projection of $\mathsf{EABox}(\T,\A)$
  on $\Sigma$. Assume that $A(a) \in \mathsf{EABox}(\T_3, \A_{\M_c,D})$ such
  that $A(a) \notin \mathsf{EABox}(\T, \A_{\M,\D})$. Then $A$ is a fresh
  concept introduced in step~3, hence $A \notin \Sigma$ and
  $\mathsf{EABox}_\Sigma(\T_3,\A_{\M_c,\D}) =
  \mathsf{EABox}_\Sigma(\T,\A_{\M,\D})$. Combining it with (II), we conclude
  that $\mathsf{EABox}_\Sigma(\T,\A_{\M_c,\D}) =
  \mathsf{EABox}_\Sigma(\T,\A_{\M,\D})$.

  \smallskip%
  Now, by Lemma~\ref{lem:rew-h-complete-aboxes} and (I), by (II) and
  by (III) we obtain the following \mbox{($\Sigma$-)}homomorphic equivalences
  $\equiv$ ($\equiv_\Sigma$):
  \[
  \C_{\tup{\T_r,\A_{\M_c,\D}}} \equiv \C_{\tup{\T_3, \A_{\M_c,\D}}}
  \equiv_\Sigma \C_{\tup{\T, \A_{\M_c,\D}}} \equiv_\Sigma \C_{\tup{\T,
      \A_{\M,\D}}}.
  \]
  Hence, $\C_{\tup{\T_r,\A_{\M_c,\D}}}$ and $\C_{\tup{\T,
      \A_{\M,\D}}}$ are $\Sigma$-homomorphically equivalent.

  Finally, by Lemma~\ref{lem:inseparability-characterization} and \textbf{(a)},
  \textbf{(b)}, we conclude that $\tup{\T,\M,\S}$ and $\tup{\T_r,\M_c,\S}$ are
  $\Sigma$-CQ inseparable.
\end{proof}

\subsection{Proofs of Section~\ref{sec:approximation}}

\begin{theoremnum}{\ref{thm:cq-rewritability}}
  Let $\tup{\T,\M,\S}$ be an OBDA specification such that $\T$ is a \halchiq
  TBox. Further, let $\T_r=\rew(\T)$ and $\M'=\cut^\Omega_k(\compile(\T,\M))$,
  for a boun\-ded\-ness oracle $\Omega$ and some $k>0$.
  If $\T$ is FO-rewritable for AQs, then $\tup{\T,\M,\S}$ is
  \emph{CQ-rewritable} into \dlliter, and $\tup{\T_r,\M',\S}$ is its
  \emph{CQ-rewriting}.  Otherwise, $\tup{\T_r,\M',\S}$ is a \emph{sound
   CQ-approximation} of $\tup{\T,\M,\S}$ in \dlliter.
\end{theoremnum}
\begin{proof}
  We note that since $\T$ is a TBox of depth 1 (it is assumed to be in normal
  form), if $\T$ is FO-rewritable for AQs then by \cite[Lemma~5]{LuWo11} $\T$
  is FO-rewritable for CQs.
  %
\end{proof}

\begin{theoremnum}{\ref{thm:rewritability-fo-mapping-undecidable}}
  The problem of checking whether an OBDA specification with an \el ontology
  and FO source queries in the mapping is CQ-rewritable into \dlliter is
  undecidable.
\end{theoremnum}

\begin{proof}
  Proof by reduction from the satisfiability problem of first-order logic.

  Let $\varphi$ be a closed first-order formula. We construct an OBDA
  specification $\P = \tup{\T,\M,\S}$ such that $\P$ is CQ-rewritable into
  \dlliter iff $\varphi$ is unsatisfiable.

  We set $\T = \{\SOME{R}{A} \ISA A\}$. $\S$ contains all predicates in
  $\varphi$, a binary relation $tableR$ and a unary relation $tableA$ such that
  $tableR, tableA$ do not occur in $\varphi$. $\M$ consists of two mapping
  assertions: $tableR(x,y) \wedge \varphi \rightsquigarrow R(x,y)$ and
  $tableA(x) \rightsquigarrow A(x)$.

  Assume that $\varphi$ is unsatisfiable. Then for each database instance $\D$
  of $\S$ we have that $R$ is empty in $\A_{\M,\D}$. It is straightforward to
  see that $\tup{\emptyset,\M,\S}$ is a CQ-rewriting of $\P$ into \dlliter.

  Assume that $\varphi$ is satisfiable and, for the sake of contradiction,
  suppose that $\P$ is CQ-rewritable into \dlliter and $\P' = \tup{\T', \M', \S}$
  is such a CQ-rewriting where $\T'$ is a \dlliter TBox. Now, consider an instance
  of the reachability problem $G = (V,E)$ and two vertices $s,t\in V$. Let $\D$
  be a database instance that satisfies $\varphi$ and such that for each $(v,u)
  \in E$, $(v,u) \in tableR$ and $s \in tableA$. It is the standard reduction
  of the reachability problem to query answering in \el, therefore $t \in
  \cert(A(x), \tup{\P,\D})$ iff $t$ is reachable from $s$ in $G$. Since $\P'$
  is a rewriting of $\P $
  and $\T'$ is a \dlliter TBox, there exists a FO-query $q_A(x)$ such that
  $\cert(A(x), \tup{\P,\D}) = \cert(A(x), \tup{\P',\D}) = \ans(q_A(x),
  \D)$. Thus, we obtain that $t \in \ans(q_A(x), \D)$ iff $t$ is reachable from
  $s$ in $G$. It means that we can solve the reachability problem by evaluating
  a FO-query over the database encoding the graph, which contradicts the
  \textsc{NLogSpace}-hardness of the reachability problem. Contradiction rises
  from the assumption that $\P$ is CQ-rewritable into \dlliter.
\end{proof}

\begin{theoremnum}{\ref{thm:rewritability-ucq-mapping-decidable}}
  The problem of checking whether an OBDA specification with a
  Horn-$\mathcal{ALCHI}$ ontology of depth one and unions of CQs as source
  queries in the mapping is CQ-rewritable into \dlliter is decidable.
\end{theoremnum}

\begin{proof}
  Let $\P = \tup{\T,\M,\S}$ be an OBDA specification (here we do not split $\M$
  into high- and low-level mappings). We construct a monadic Datalog program
  $\Pi$ without inequalities worst case exponential in the size of $\T$ and
  $\M$ such that
  \begin{equation}
    \label{eq:bounded-rewritable}
    \text{$\Pi$ is program bounded iff $\P$ is CQ-rewritable into \dlliter},
  \end{equation}
  where $\Pi$ is said to be program bounded if each predicate $N$ mentioned in
  $\Pi$ is bounded in $\Pi$, and a Datalog program $\Pi$ is \emph{monadic} if
  all its IDB predicates are monadic (unary). It is known that program
  boundedness of monadic Datalog programs without inequalities in decidable in
  3\textsc{ExpTime} \cite{CGKV88}. Thus, we obtain a 4\textsc{ExpTime}
  algorithm for deciding CQ-rewritability into \dlliter.


  Let $\T_3$ be the TBox obtained as an intermediate result in Step~3 of
  $\rewcomp(\T,\M)$.
  Then $\Pi$ is the monadic Datalog program such that
  \begin{align}\label{eq:pi-property}
    &A_{\Pi}^\infty(\D) = A_{\Pi_{\T_3,\M}}^\infty(\D),\\
    &\notag\qquad\text{for each instance $\D$ of $\S$ and each concept $A$ in $\T_3$,}
  \end{align}
  and $\varphi^A$ is DB-defined for each $\varphi^A \in \Phi_\Pi(A)$.  Observe
  that the Datalog translation $\Pi_{\T_3,\M}$ of the Horn-$\mathcal{ALCHI}$
  TBox $\T_3$ is a Datalog program without inequalities. Therefore, we have
  that $\Pi$ is a monadic Datalog program without inequalities and its program
  boundedness is decidable. Moreover, observe that $\Phi_{\Pi_1}(P)$ is a
  finite union of CQs for each role name $P$ in $\T_3$.
  We first prove \eqref{eq:bounded-rewritable}, then we show how $\Pi$ is
  constructed.

  Assume that $\Pi$ is program bounded and let $\Omega$ be a boundedness oracle
  for it. Note that since $\Pi$ is bounded, for a concept name $A$,
  $\cut_{k}^\Omega(A, \Pi)$ does not depend on the value of $k$. Next, let
  $\T_r = \rew(\T)$, and $\M_c$ be the set of
  \begin{compactitem}
  \item mapping assertions $\varphi^A(x) \rightsquigarrow A(x)$ such that $A$
    is a concept name in~$\T_3$ and $\varphi^A\in\cut_{k}^\Omega(A, \Pi)$, and of
  \item mapping assertions $\varphi^P(x,y) \rightsquigarrow P(x,y)$ such that
    $P$ is a role name in~$\T_3$ and $\varphi^P\in\Phi_{\Pi_{\T_3,\M}}(P)$.
  \end{compactitem}
  It is straightforward to see that $\tup{\T_r, \M_c, \S}$ is a CQ-rewriting of
  $\tup{\T,\M,\S}$ into \dlliter.

  Assume that $\Pi$ is CQ-rewritable into \dlliter and $\tup{\T', \M', \S}$ is
  its CQ-rewriting where the source queries in $\M'$ are unions of CQs. Let $N$
  be a concept or role name in $\T$ and denote by $q_N(\vec{x})$ the rewriting
  of the query $N(\vec{x})$ into a union of CQs over $\S$ with respect to
  $\tup{\T', \M', \S}$ (recall that the rewriting of $N(\vec{x})$ with respect
  to $\T'$ is a union of CQs, and since $\M'$ contains unions of CQs as source
  queries, $q_N(\vec{x})$ is also a union of CQs).

  We construct now a Datalog program $\Pi'$ consisting of the rules $N(\vec{x})
  \leftarrow \varphi^N(\vec{x})$, for a concept or role name $N$ in $\T$ and a
  CQ $\varphi^N(\vec{x}) \in q_N(\vec{x})$. Obviously, $\Pi'$ is program
  bounded. Since $\tup{\T', \M', \S}$ is a CQ-rewriting of $\tup{\T,\M,\S}$,
  and $\T_3$ is a model-conservative extension of $\T$, we have that
  $N_{\Pi'}^\infty(\D) = N_{\Pi_{\T,\M}}^\infty(\D) =
  N_{\Pi_{\T_3,\M}}^\infty(\D)$, for each instance $\D$ of $\S$ and each
  concept or role name $N$ in $\T$. Next, because of \eqref{eq:pi-property}, we
  have that $A_{\Pi'}^\infty(\D) = A_{\Pi}^\infty(\D)$ for each instance $\D$
  of $\S$ and each concept name $A$ in $\T$. Now, we set the finite union of
  CQs $\Omega_\Pi(A)$ for each concept name $A$ in $\T_3$:
  \begin{compactitem}
  \item if $A$ is a concept name in $\T$, then $\Omega_\Pi(A) =
    \Phi_{\Pi'}(A)$
  \item otherwise, $A$ is introduced for a concept conjunction $A_1 \AND \cdots
    \AND A_n$ in Step~3, then $\Omega_\Pi(A)$ is the DNF of the formula
    $\Phi_{\Pi'}(A_1) \land \cdots \land \Phi_{\Pi'}(A_n)$ where each
    $\Phi_{\Pi'}$ is viewed as a formula in DNF.
  \end{compactitem}
  Hence, we obtain that $\Pi$ is program bounded.

  \medskip%
  We now show how $\Pi$ is constructed from $\Pi_{\T_3,\M}$.

  First, we remove from $\Pi_{\T_3,\M}$ the rules which are not reachable from
  the database predicates. Namely, let $\Pi_1$ be the set of rules $\pi
  =\mathit{head} \leftarrow X_1, \dots, X_n$ in $\Pi_{\T_3,\M}$ such that there
  are sets of rules $\rho_1, \dots, \rho_m$ in $\Pi_{\T_3,\M}$ such that
  $\rho_m$ is a set of rules from $\Pi_\M$, the predicates in the bodies of the
  rules in $\rho_{i-1}$ are exactly the predicates in the heads of the rules in
  $\rho_i$, for $2 \leq i\leq m$, and $\rho_1 = \{\pi\}$. It should be clear
  that $N_{\Pi_1}^\infty(\D) = N_{\Pi_{\T_3,\M}}^\infty(\D)$, for each instance
  $\D$ of $\S$ and each concept or role name $N$ in $\T_3$.

  Then $\Pi$ is the monadic Datalog program such that for each instance $\D$ of
  $\S$ and each concept name $A$ in $\T_3$, $A_{\Pi}^\infty(\D) =
  A_{\Pi_1}^\infty(\D)$. We obtain $\Pi$ by substituting each occurrence in the
  body of a rule of an atom of the form $R(x,y)$, for $R$ a role in $\T$, by
  $\Phi_{\Pi_1}(R)$, and by removing all rules whose head predicates are roles.
  Namely for the former, let $\rho = \mathit{head} \leftarrow \varphi, R(x,y)$
  be a rule in $\Pi_1$. Then we replace $\rho$ with the rules, $\mathit{head}
  \leftarrow \varphi, \psi$, for each CQ $\psi \in \Phi_{\Pi_1}(R)$. We repeat
  this procedure until we get a Datalog program $\Pi$ where no atom of the form
  $R(x,y)$, for a role $R$ in $\T_3$, occurs in the a body of a rule. Observe
  that $\Phi_{\Pi_1}(R)$ is always finite. 

  It is easy to see that $\Pi$ is as required.
\end{proof}


\fi

\end{document}